\documentclass[11pt]{article} 

\usepackage{fullpage}

\usepackage{amsmath,amsthm, amssymb}
\usepackage{latexsym}
\usepackage{graphicx}
\usepackage{ifthen}
\usepackage{subcaption}
\usepackage{multicol}
\usepackage{algorithm, algorithmic}
\usepackage[numbers]{natbib}
\usepackage{mathtools}
\usepackage{dsfont}
\usepackage{nicefrac}
\usepackage{subcaption}

\usepackage{hyperref}
\hypersetup{colorlinks=true,linkcolor=blue,citecolor=blue,urlcolor=blue}

\newtheorem*{theorem*}{Theorem}
\newtheorem*{definition*}{Definition}

\newtheorem{lemma}{Lemma}
\newtheorem{theorem}{Theorem}

\newtheorem{corollary}{Corollary}

\usepackage{thm-restate}

\renewcommand{\algorithmiccomment}[1]{\bgroup\hfill\footnotesize~#1\egroup}

\newcommand{\algone}{{\textsc{Adaptive-Sequencing}}}
\newcommand{\algoptimized}{{\textsc{Fast}}}
\newcommand{\algfull}{{\textsc{Fast-Full}}}

\newcommand{\sample}{\textsc{Sample}}
\newcommand{\sequence}{\textsc{Sequence}}
\newcommand{\geom}{\textsc{Geometric}}

\DeclareMathOperator{\E}{\mathbb{E}}
\newcommand{\EU}[1][]{\underset{#1}{\E}}

\DeclareMathOperator{\poly}{poly}
\newcommand{\OPT}{\texttt{OPT}}

\renewcommand{\O}{\mathcal{O}}

\title{The \textsc{Fast} Algorithm for Submodular Maximization}

\author{
  Adam Breuer\\
  Harvard University\\
  \texttt{breuer@g.harvard.edu} 
  \and  Eric Balkanski \\
  Harvard University\\
  \texttt{ericbalkanski@g.harvard.edu} 
\and
  Yaron Singer \\
  Harvard University\\
  \texttt{yaron@seas.harvard.edu}
}
\date{}

\begin{document}

\setcounter{page}{0}

\maketitle

\begin{abstract}
In this paper we describe a new algorithm called Fast Adaptive Sequencing Technique (\textsc{Fast}) for maximizing a monotone submodular function under a cardinality constraint $k$ whose approximation ratio is arbitrarily close to $1-1/e$, is $\O(\log(n) \log^2(\log k))$ adaptive, and uses a total of $\O(n \log\log(k))$ queries.  Recent algorithms have comparable guarantees in terms of asymptotic worst case analysis, but their actual number of rounds and query complexity depend on very large constants and polynomials in terms of precision and confidence, making them impractical for large data sets. Our main contribution is a design that is extremely efficient both in terms of its non-asymptotic worst case query complexity and number of rounds, and in terms of its practical runtime.  We show that this algorithm outperforms any algorithm for submodular maximization we are aware of, including hyper-optimized parallel versions of state-of-the-art serial algorithms, by running experiments on large data sets. These experiments show \algoptimized \ is orders of magnitude faster than the state-of-the-art. 
\end{abstract}

\newpage


\section{Introduction}
In this paper we describe a fast parallel algorithm for submodular maximization.  Informally, a function is submodular if it exhibits a natural diminishing returns property.  For the canonical problem of maximizing a monotone submodular function under a cardinality constraint, it is well known that the greedy algorithm, which iteratively adds elements whose marginal contribution is largest to the solution, obtains a $1-1/e$ approximation guarantee~\cite{NWF78} which is optimal for polynomial-time algorithms~\cite{nemhauser1978best}.  The greedy algorithm and other submodular maximization techniques are heavily used in machine learning and data mining since many  fundamental objectives such as  entropy, mutual information, graphs cuts, diversity, and set cover are all submodular.  

In recent years there has been a great deal of progress on fast algorithms for submodular maximization designed to accelerate computation on large data sets.  The first line of work considers \emph{serial} algorithms where queries can be evaluated on a single processor~\cite{LKGFVG07,BV14,mirzasoleiman2015lazier,mirzasoleiman2016fast,EN19-1,EN19-2}.  For serial algorithms the state-of-the-art for maximization under a cardinality constraint is the \emph{lazier-than-lazy-greedy} (\textsc{LTLG}) algorithm which returns a solution that is in expectation  arbitrarily close to the optimal $1-1/e$  and does so in a linear number of queries~\cite{mirzasoleiman2015lazier}.  This algorithm is a stochastic greedy algorithm coupled with lazy updates, which not only performs well in terms of the quality of the solution it returns but is also very fast in practice.  
  
Accelerating computation beyond linear runtime requires \emph{parallelization}.  The parallel runtime of blackbox optimization is measured by \emph{adaptivity}, which is the number of sequential rounds an algorithm requires when polynomially-many queries can be executed in parallel in every round.  For maximizing a submodular function defined over a ground set of $n$ elements under a cardinality constraint $k$, the adaptivity of the naive greedy algorithm is $\O(k)$, which in the worst case is $\O(n)$.  Until recently no algorithm was known to have better parallel runtime than that of naive greedy. 

A very recent line of work initiated by Balkanski and Singer~\cite{BS18a} develops techniques for designing constant factor approximation algorithms for submodular maximization whose parallel runtime is \emph{logarithmic}~\cite{BS18b,BRS19,EN19,FMZ19,FMZ18,BBS18,KMZLK19,chekuri2018submodular,BRS19b,chekuri2018matroid,ene2019,chen2019,FMZ19}.  In particular, ~\cite{BS18a} describe a technique called \emph{adaptive sampling} that obtains in $\O(\log n)$ rounds an approximation arbitrarily close to $1/3$ for maximizing a monotone submodular function under a cardinality constraint. This technique can be used to produce solutions arbitrarily close to the optimal $1-1/e$ in $\O(\log n)$ rounds~\cite{BRS19,EN19}.

\subsection{From theory to practice}
The focus of the work on adaptive complexity described above has largely been on conceptual and theoretical contributions: achieving strong approximation guarantees under various constraints with runtimes that are exponentially faster under worst case theoretical analysis.  From a practitioner's perspective however, even the state-of-the-art algorithms in this genre are infeasible for large data sets.  The logarithmic parallel runtime of algorithms in this genre carries extremely large constants and polynomial dependencies on  precision and confidence parameters that are hidden in their asymptotic analysis.  In terms of sample complexity alone, obtaining (for example) a $1-1/e-0.1$ approximation with $95\%$ confidence for maximizing a submodular function under cardinality constraint $k$ requires evaluating at least $10^8$~\cite{BRS19} or $10^6$~\cite{FMZ19, chekuri2018submodular} samples of sets of size approximately $\frac{k}{\log n}$ in every round.  Even if one heuristically uses a single sample in every round, other sources of inefficiencies that we discuss throughout the paper that prevent these algorithms from being applied even on moderate-sized data sets.  The question is then whether the plethora of breakthrough techniques in this line of work of exponentially faster algorithms for submodular maximization can lead to algorithms that are fast in practice for large problem instances.

\subsection{Our contribution}
In this paper we design a new algorithm called Fast Adaptive Sequencing Technique (\textsc{Fast}) whose approximation ratio is arbitrarily close to $1-1/e$, has $\O(\log(n) \log^2(\log k))$ adaptivity, and uses a total of $\O(n \log\log(k))$ queries for maximizing a monotone submodular function under a cardinality constraint $k$.  
The main contribution is not in the algorithm's asymptotic guarantees but in its design that is extremely efficient both in terms of its non-asymptotic worst case query complexity and number of rounds, and in terms of its practical runtime.  In terms of \emph{actual} query complexity and practical runtime, this algorithm outperforms any algorithms for submodular maximization we are aware of, including hyper-optimized versions of \textsc{LTLG}.  To be more concrete, we give a brief experimental comparison in the table below for a max-cover objective on a Watts-Strogatz random graph with $n=500$ nodes against optimized implementations of algorithms with the same adaptivity and approximation  (experiment details in Section~\ref{sec:experiments}).  

\begin{center}
\begin{tabular}{llllll}\label{tab:comp}
                        &   & \emph{rounds} & \emph{queries} & \emph{time (sec)}  &  \\
\textsc{Amortized-Filtering}&~\cite{BRS19}        & 540    & 35471   & 0.58  &  \\
\textsc{Exhaustive-Maximization}&~\cite{FMZ19}    & 12806  & 4845205 & 55.14 &  \\
\textsc{Randomized-Parallel-Greedy}&~\cite{chekuri2018submodular} & 66   & 81648  & 1.36  &  \\
\textsc{Fast}                   &    & \textbf{18}     & \textbf{2497}    & \textbf{0.051} &  \\
\end{tabular}
\end{center}

\algoptimized \ achieves its speedup by thoughtful design that results in frugal worst case query complexity as well as several heuristics used for practical speedups.  From a purely analytical perspective, \algoptimized \ improves the $\varepsilon$ dependency in the linear term of the query complexity of at least $\O(\varepsilon^{-5}n)$ in~\cite{BRS19, EN19} and $\O(\varepsilon^{-3}n)$ in~\cite{FMZ19} to $\O(\varepsilon^{-2} n)$.  We provide the first non-asymptotic bounds on the query and adaptive complexity  of an algorithm with sublinear adaptivity, showing dependency on small constants. In Appendix~\ref{sec:appcomparison}, we compare these query and adaptive  complexity achieved by \algoptimized \ to previous work. Our algorithm uses adaptive sequencing~\cite{BRS19b} and multiple optimizations to improve the  query complexity and runtime. 

\subsection{Paper organization}
We introduce the main ideas and decisions behind the design of \algoptimized \ in Section~\ref{sec:cardinality}.  We describe and analyze guarantees in Section~\ref{sec:algorithm}.  We discuss experiments in Section~\ref{sec:experiments}.


\section{\algoptimized \ Overview}
\label{sec:cardinality}

Before describing the algorithm, we give an overview of the major ideas and discuss how they circumvent the bottlenecks for practical implementation of existing logarithmic adaptivity algorithms.  

\paragraph{Adaptive sequencing vs. adaptive sampling.} The large majority of low-adaptivity algorithms use \emph{adaptive sampling}~\cite{BS18a, EN19,FMZ19,FMZ18,BBS18,BRS19,KMZLK19}, a technique introduced in~\cite{BS18a}.  These algorithms sample a large number of sets of elements at every iteration to estimate (1) the expected contribution of a random set $R$ to the current solution $S$ and (2) the expected contributions of each element $a$ to $R \cup S$. These estimates, which rely on concentration arguments, are then used to either add a random set $R$ to $S$ or  discard elements with low expected contribution to $R \cup S$. In contrast, the \emph{adaptive sequencing} technique which was recently introduced in \cite{BRS19b} generates at every iteration  a \emph{single} random sequence $(a_1, \ldots, a_{|X|})$ of the elements $X$ not yet discarded.  A prefix $A_{i^{\star}} = (a_1, \ldots, a_{i^{\star}})$ of the sequence is then added to the solution $S$, where $i^{\star}$ is the largest position $i$ such that a large fraction of the elements in $X$ have high contribution to $S \cup A_{i-1}$. Elements with low contribution to the new solution $S$ are then discarded from $X$.

The first choice we made was to use an adaptive sequencing technique rather than adaptive sampling.  

\begin{itemize}

\item  \textbf{Dependence on large polynomials in $\varepsilon$.} Adaptive sampling algorithms crucially rely on sampling and as a result their query complexity has high polynomial dependency on $\varepsilon$ (e.g. at least $\O(\epsilon^{-5}n)$ in~\cite{BRS19} and  \cite{EN19}). In contrast, adaptive \emph{sequencing} generates a \emph{single} random sequence at every iteration and we can therefore obtain an $\varepsilon$ dependence in the term that is linear in $n$ that is only $\O(\varepsilon^{-2} n)$;

\item \textbf{Dependence on large constants.} The asymptotic query complexity of previous algorithms depends on very large constants (e.g. at least $60000$ in~\cite{BRS19} and  \cite{EN19}) making them impractical.  As we tried to optimize constants for adaptive sampling, we found that due to the sampling and the requirement to maintain strong theoretical guarantees, the constants cascade and grow through multiple parts of the analysis.  In principle, adaptive \emph{sequencing} does not rely on sampling which dramatically reduces the dependency on constants. 
\end{itemize}

\paragraph{Negotiating adaptive complexity with query complexity.}  The vanilla version of our algorithm, whose description and analysis are in Appendix~\ref{sec:algone}, has at most $\varepsilon^{-2}  \log n$ adaptive rounds and  uses a total of $\varepsilon^{-2} n k$ queries to obtain a $1-1/e - \frac{3}{2}\epsilon$ approximation, without additional dependence on constants or lower order terms.  In our actual algorithm, we trade a small factor in adaptive complexity 
for a substantial improvement in query complexity.  We do this in the following manner:
%
%
\begin{itemize}
\item \textbf{Search for $\OPT$ estimates.} All algorithms with logarithmic adaptivity require a good estimate of $\OPT$, which can be obtained by running $\varepsilon^{-1} \log k$ instances of the  algorithms with different guesses of $\OPT$ in parallel, so that one guess is guaranteed to be a good approximation to $\OPT$.\footnote{\cite{FMZ19} does some preprocessing to estimate $\OPT$, but it is estimated within some very large constant.}  We accelerate this search by binary searching over the guesses of $\OPT$. A main difficulty for this binary search  is that the approximation guarantee of each solution needs to hold with high probability instead of in expectation. Even though the marginal contributions obtained from each element added to the solution only hold in expectation for adaptive sequencing, we obtain high probability guarantees for the global solution by generalizing the robust guarantees of \cite{HS17}. In the practical speedups below, we discuss how we often only need a single iteration of this binary search in practice;
\item \textbf{Search for position $i^{\star}$.} To find the position $i^{\star}$, which is the largest position $i \in [k]$ such that a large fraction of not-yet-discarded elements have high contribution to $S \cup A_{i-1}$,  the vanilla adaptive sequencing technique queries the contribution of all  elements in $X$ at each of the $k$ positions, which causes the $\O(nk)$ query complexity. Instead, similar to guessing $\OPT$, we binary search over a set of $\varepsilon^{-1} \log k$ geometrically increasing values of $i$. This improves the $\O(nk)$ dependency on $n$ and $k$ in the query complexity to $\O(n\log(\log k))$. Then, at any step of the binary search over a position $i$, instead of evaluating the contribution of all elements in $X$ to $S \cup A_{i-1}$, we only evaluate a small sample of elements. In the practical speedups below, we discuss how we can often skip this binary search for $i^{\star}$ in practice.
\end{itemize}

\paragraph{Practical speedups.}  We include several ideas which result in considerable speedups in practice without sacrificing approximation, adaptivity, or query complexity guarantees:

\begin{itemize}
\item \textbf{Preprocessing the sequence.} At the outset of each iteration of the algorithm, before searching for a prefix $A_{i^{\star}}$ to add to the solution $S$, we first use a preprocessing step that adds high value elements from the sequence to $S$. Specifically, we add to the solution $S$ all sequence elements $a_i$ that have high contribution to $S \cup A_{i-1}$. After adding these high-value elements, we discard surviving elements in $X$ that have low contribution to the new solution $S$. In the case where this step discards a large fraction of surviving elements from $X$, we can also skip this iteration's binary search for $i^{\star}$ and continue to the next iteration without adding a prefix to $S$;

\item \textbf{Number of elements added per iteration.} An adaptive sampling algorithm which samples sets of size $s$ adds at most $s$ elements to the current solution at each iteration. In contrast, adaptive sequencing and the preprocessing step described above often allow our algorithm to add a very large number of elements to the current solution at each iteration in practice;
\item \textbf{Single iteration of the binary search for $\OPT$.} Even with binary search, running multiple instances of the algorithm with different guesses of $\OPT$ is undesirable. We describe a technique which often needs only a single guess of $\OPT$. This guess is the sum $v = \max_{|S| \leq k} \sum_{a \in S} f(a)$  of the $k$ highest valued singletons, which is an upper bound on $\OPT$.  If the solution $S$ obtained with that guess $v$ has value $f(S) \geq (1-1/e)v$, then, since $v \geq \OPT$, $S$ is guaranteed to obtain a $1-1/e$  approximation and the algorithm does not need to continue the binary search. Note that with a single guess of $\OPT$, the robust guarantees for the binary search are not needed, which improves the sample complexity to $m = \frac{2 + \varepsilon}{\varepsilon^2(1 -  3\varepsilon)} \log(2\delta^{-1})$;
\item \textbf{Lazy updates.} There are many situations where lazy evaluations  of marginal contributions can be performed \cite{minoux1978accelerated, mirzasoleiman2015lazier}. Since we never discard elements from the solution $S$, the contributions of elements $a$ to $S$ are non-increasing at every iteration by submodularity. Elements with low contribution $c$ to the current solution at some iteration are ignored until the threshold $t$ is lowered to $t \leq c$. Lazy updates also accelerate the binary search over $i^{\star}$.
\end{itemize}

\section{The Algorithm}\label{sec:algorithm}

We describe the  \algfull \ algorithm (Algorithm~\ref{alg:full}). The main part of the algorithm is the \algoptimized \ subroutine (Algorithm~\ref{alg:main}), which is instantiated with different guesses of $\OPT$. These guesses $v \in V$ of $\OPT$ are geometrically increasing from $\max_{a \in N} f(a)$ to $\max_{|S| \leq k} \sum_{a \in S} f(a)$ by a $(1 - \varepsilon)^{-1}$ factor, so $V$ contains a value that is a $1- \varepsilon$ approximation to $\OPT$. The algorithm binary searches over guesses for the largest guess $v$ that obtains a solution $S$ that is a $1-1/e$ approximation to $v$.

\begin{algorithm}[H]
\caption{\algfull: the full algorithm}
\begin{algorithmic}
    	\STATE \textbf{input} function $f$, cardinality constraint $k$, parameter $\varepsilon$
    	\STATE $V \leftarrow \geom(\max_{a \in N} f(a), \max_{|S| \leq k} \sum_{a \in S} f(a), 1 - \varepsilon)$
    	\STATE $v^{\star} \leftarrow$ \textsc{Binary-Search}$(V, \max\{v \in V: f(S_v) \geq (1 - 1/e)v\})$ where
    	\STATE \ \ \ \ \ $S_v \leftarrow \algoptimized(v)$
    	\STATE \textbf{return} $S_{v^{\star}}$ 
  \end{algorithmic}
  \label{alg:full}
\end{algorithm}

 \algoptimized \  generates at every iteration a  uniformly random sequence $a_1, \ldots, a_{|X|}$ of the elements $X$ not yet discarded. After the preprocessing step which adds to $S$ elements guaranteed to have high contribution, the  algorithm  identifies a position $i^{\star}$ in this sequence which determines the prefix $A_{i^{\star} - 1}$ that is added to the current solution $S$. Position $i^{\star}$ is defined as the largest position such that there is a large fraction of elements in $X$ with high contribution to $S \cup A_{i^{\star} - 1}$. To find $i^{\star}$, we binary search over geometrically increasing positions $i \in I \subseteq [k]$. At each position $i$, we only evaluate the contributions of elements $a \in R$ , where $R$ is a uniformly random subset of $X$ of size $m$, instead of all elements $X$.  
 
 \begin{algorithm}[H]
\caption{\algoptimized: the Fast Adaptive Sequencing Technique algorithm}
\begin{algorithmic}
    	\STATE \textbf{input}  function $f$, cardinality constraint $k$, guess $v$ for $\OPT$, parameter $\varepsilon$
    	\STATE 	 $S \leftarrow \emptyset $
    	\STATE  \textbf{while} $|S| < k$ and number of  iterations  $< \varepsilon^{-1}$ \textbf{do}
    	\STATE   \ \ \ \ \  $X \leftarrow N,  t \leftarrow (1-\varepsilon) (v - f(S))/k$
	\STATE   \ \ \ \ \  \textbf{while} $X \neq \emptyset$ and $|S| < k$  \textbf{do}
		\STATE  \ \ \ \ \   \ \ \ \ \   $a_1, \ldots, a_{|X|} \leftarrow \sequence(X, |X|), A_i \leftarrow a_1, \ldots, a_i$
	\STATE \ \ \ \ \   \ \ \ \ \   $S \leftarrow S \cup \{a_i : f_{S \cup A_{i-1}}(a_i) \geq t \}$
	\STATE \ \ \ \ \   \ \ \ \ \ $X_0 \leftarrow \{a \in X : f_{S}(a) \geq t\}$
	\STATE   \ \ \ \ \   \ \ \ \ \   \textbf{if} $|X_0| \leq (1 - \varepsilon) |X|$ \textbf{then} $X \leftarrow X_0$ and continue to  next iteration
	\STATE    \ \ \ \ \   \ \ \ \ \  $R \leftarrow \sample(X, m)$, $I \leftarrow \geom(1, k - |S|, 1 - \varepsilon)$
\STATE  \ \ \ \ \  \ \ \ \ \   $i^{\star} \leftarrow$ \textsc{Binary-Search}$(I,  \max \{i  \in I :  |\left\{a \in R : f_{S \cup A_{i-1}}(a) \geq   t \right\}| \geq (1 - 2 \varepsilon) |R|\})$ 
	\STATE  \ \ \ \ \   \ \ \ \ \    $S \leftarrow S \cup A_{i^{\star} }$
    	\STATE \textbf{return} $S$
  \end{algorithmic}
  \label{alg:main}
\end{algorithm}

\subsection{Analysis}

We show that \algoptimized \ obtains a $1-1/e-\epsilon$ approximation w.p. $1- \delta$ and that it has $\tilde{O}(\varepsilon^{-2} \log n)$  adaptive complexity and $\tilde{O}(\varepsilon^{-2} n +\varepsilon^{-4} \log(n)  \log(\delta^{-1}))$ query complexity.

\begin{restatable}{rThm}{thmmain}
\label{thm:main} 
Assume $k \geq \frac{2 \log(2\delta^{-1} \ell)}{\varepsilon^2 (1 - 5\varepsilon)}$ and $\varepsilon \in (0, 0.1)$, where  $\ell = \log(\frac{\log k}{\epsilon})$. Then, \algoptimized \ with $m = \frac{2 + \varepsilon}{\varepsilon^2(1 -  3\varepsilon)} \log(\frac{4\ell\log n}{\delta \varepsilon^2})$ has at most $\varepsilon^{-2} \log(n)  \ell^2$  adaptive rounds,   
$2\varepsilon^{-2}  \ell n +  \varepsilon^{-2} \log(n) \ell^2 m$ queries, and achieves a $1 - \frac{1}{e} - 4 \varepsilon$ approximation 
w.p. $1 - \delta$.
\end{restatable}

We defer the analysis to Appendix~\ref{sec:appcardinality}. The main part of it is for the approximation guarantee, which consists of two cases depending on the condition which breaks the outer-loop. Lemma~\ref{lem:eps2} shows that when there are $\varepsilon^{-1}$ iterations of the outer-loop,   the set of elements  added to  $S$ at every iteration of the outer-loop contributes $\varepsilon^{-1}(\OPT - f(S))$. Lemma~\ref{lem:marg2} shows that for the case where $|S| = k$, the expected  contribution of each element $a_i$ added to  $S$ is arbitrarily close to $(\OPT - f(S))/k$. For each solution $S_v$, we need the approximation guarantee to hold with high probability instead of in expectation to be able to binary search over guesses for \OPT, which we obtain in Lemma~\ref{lem:main} by generalizing  the robust guarantees of \cite{HS17} in Lemma~\ref{lem:HS17}. The main observation to obtain the adaptive complexity (Lemma~\ref{lem:adaptivity2}) is that, by definition of $i^{\star}$, at least  an $\varepsilon$ fraction of the surviving elements in $X$ are discarded at every iteration with high probability.\footnote{To obtain the adaptivity $r$ with probability $1$ and the approximation guarantee w.p. $1-\delta$, the algorithm declares failure  after $r$ rounds and accounts for this failure probability in $\delta$.} For the query complexity (Lemma~\ref{lem:query2}), we note that there are  $|X| + m \ell$ function evaluations per iteration.


\section{Experiments}\label{sec:experiments}
Our goal in this section is to show that in practice, \algoptimized \ finds solutions whose value meets or exceeds alternatives in less parallel runtime than both state-of-the-art low-adaptivity algorithms and \textsc{Lazier-than-Lazy-Greedy}. 
To accomplish this, we build optimized parallel MPI implementations of \algoptimized, other  low-adaptivity algorithms, and \textsc{Lazier-than-Lazy-Greedy}, which is widely regarded as the fastest algorithm for submodular maximization in practice. We then use $95$ Intel Skylake-SP $3.1$ GHz processors on AWS to compare the algorithms' runtime over a variety of objectives defined on $8$ real and synthetic datasets. We measure runtime using a rigorous measure of parallel time (see Appendix \ref{ssec:measure_ptime}). Appendices \ref{ssec:experiments_algos}, \ref{ssec:datasets}, \ref{ssec:implementations}, and \ref{ssec:hardware} contain detailed descriptions of the benchmarks, objectives, implementations, hardware, and experimental setup on AWS. 


We conduct two sets of experiments. The first set compares \algoptimized \ to previous low-adaptivity algorithms. Since these algorithms  all have practically intractable sample complexity, we grossly reduce their sample complexity to only $95$ samples per iteration so that each processor performs a single function evaluation per iteration. This reduction, which we discuss in detail below, gives these algorithms a large runtime advantage over \algoptimized, which computes its full theoretical sample complexity in these experiments. This is practically feasible for \algoptimized \ because \algoptimized \ samples \emph{elements}, not \emph{sets} of elements like other low-adaptivity algorithms. Despite the large advantage this setup gives to the other low-adaptivity algorithms, \algoptimized \ is consistently one to three orders of magnitude faster.

The second set of experiments compares \algoptimized \ to \textsc{Parallel-Lazier-than-Lazy-Greedy} (\textsc{Parallel-LTLG}). We scale up the $8$ objectives to be defined over synthetic data with $n=100000$ and real data with up to $n=26000$ with $k$ from $25$ to $25000.$ We find that \algoptimized \ is consistently $1.5$ to $27$ times faster than \textsc{Parallel-LTLG} and that its runtime advantage increases in $k$. These fast relative runtimes are a loose lower bound on \algoptimized's performance advantage, as \algoptimized \ can reap additional speedups by adding up to $n$ processors, whereas \textsc{Parallel-LTLG} performs at most $n\log(\epsilon^{-1})/k$ function evaluations per iteration, so using over $95$ processors often does not help. In Section \ref{ssec:resultsLTLG}, we show that on many objectives \algoptimized \ is faster even with only a single processor.

\subsection{Experiment set 1: \algoptimized \ vs. low-adaptivity algorithms}
\label{ssec:adaptive_experiments}
Our first set of experiments compares \algoptimized \ to state-of-the-art low-adaptivity algorithms. To accomplish this, we built optimized parallel MPI versions of each of the following algorithms: \textsc{Randomized-Parallel-Greedy}~\cite{chekuri2018submodular}, \textsc{Exhaustive-Maximization}~\cite{FMZ19}, and \textsc{Amortized-Filtering}~\cite{BRS19}.  For any given $\varepsilon>0$ all these algorithms achieve a $1-1/e-\varepsilon$ approximation in $\O(\poly(\varepsilon^{-1})\log n)$ rounds.  For calibration, we also ran (1) \textsc{Parallel-Greedy}, a parallel version of the standard \textsc{Greedy} algorithm as a heuristic upper bound for the objective value, as well as (2) \textsc{Random}, an algorithm that simply selects $k$ elements uniformly at random. 

A fair comparison of  the algorithms' parallel runtimes and solution values is to run each algorithm with parameters that yield the same guarantees, for example a $1-1/e - \epsilon$ approximation w.p. $1 - \delta$ with $\varepsilon=0.1$ and $\delta=0.05$.  However, this is infeasible since the other low-adaptivity algorithms all require a practically intractable number of queries to achieve any reasonable guarantees, e.g.,  every round of \textsc{Amortized-Filtering} would require at least  $10^8$ samples, even with $n = 500$. 


\paragraph{Dealing with practically intractable query complexity of benchmarks.} To run other low-adaptivity algorithms despite their huge sample complexity we made two major modifications:

\begin{enumerate}
\item \textbf{Accelerating subroutines.} We optimize each of the three other low-adaptivity benchmarks by implementing parallel binary search to replace brute-force search and several other modifications that reduce unnecessary queries (for a full description of these fast implementations, see Appendix~\ref{sec:fast_as}). These optimizations result in speedups that reduce their runtimes by an order of magnitude in practice, and these optimized implementations are publicly available in our code base. Despite this, it remains practically infeasible to compute these algorithms' high number of samples in practice even on small problems (e.g. $n=500$ elements);

\begin{figure}
\centering
\includegraphics[height=0.45\textwidth]{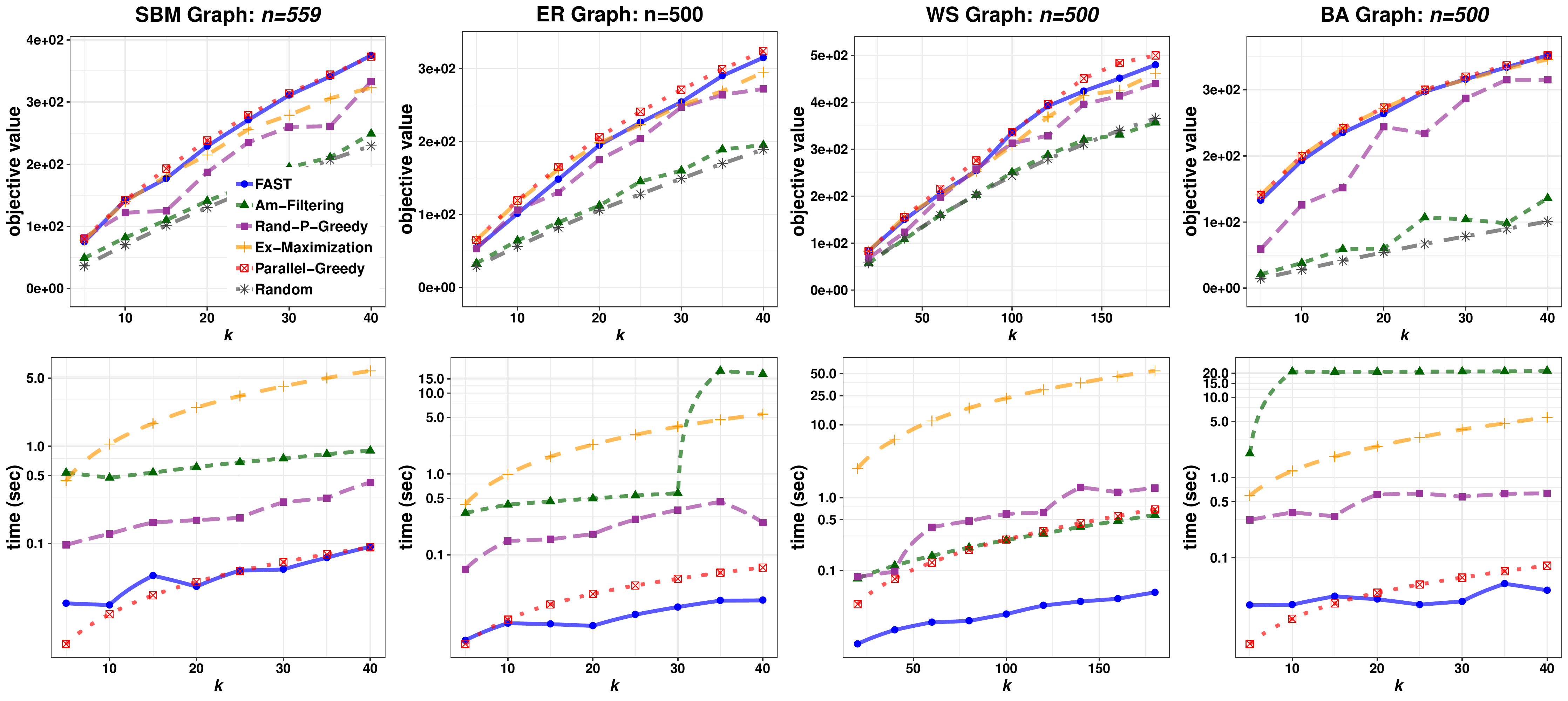}
\caption{\textit{Experiment Set 1.a: \algoptimized \ vs. low-adaptivity algorithms on graphs (time axis\textbf{ log-scaled})}.}
\label{fig:adaptive_rand_graphs}
\end{figure}

\item \textbf{Using a single query per processor.} Since our interest is in comparing runtime and not quality of approximation, we dramatically lowered the number of queries the three benchmark algorithms require to achieve their guarantees. 
Specifically, we set the parameters $\varepsilon$ and $\delta$ for both \algoptimized \ and the three low-adaptivity benchmarks such that all algorithms guarantee the same $1- 1/e - 0.1$ approximation with probability $0.95$ (see Appendix~\ref{app:epsilons}). However, for the low-adaptivity benchmarks, we reduce their theoretical  sample complexity in each round to have exactly \textbf{one} sample per processor (instead of their large sample complexity, e.g. $10^8$ samples needed for \textsc{Amortized-Filtering}). This reduction in the number of samples per round allows the benchmarks to have each processor perform a single function evaluation per round instead of e.g. $10^8/95$ functions evaluations per processor per round, which `unfairly' accelerates their runtimes at the expense of their approximations. However, we do \emph{not} perform this reduction for \algoptimized. Instead, we require \algoptimized \ to compute the \emph{full count} of samples for its guarantees. This is feasible since \algoptimized \ samples elements and not sets.
\end{enumerate}

\paragraph{Data sets.} Even with these modifications, for tractability we could only use small data sets: 
 
 \begin{itemize}
     \item \textbf{Experiments 1.a: synthetic data sets ($n\approx 500$).} To compare the algorithms' runtimes under a range of conditions, we solve max cover on synthetic graphs generated via four different well-studied graph models: \emph{Stochastic Block Model} (SBM); \emph{Erd\H{o}s R\'{e}nyi} (ER); \emph{Watts-Strogatz} (WS); and \emph{Barb\'{a}si-Albert} (BA). See Appendix~\ref{ssec:random_graphs} for additional details;
     \item \textbf{Experiments 1.b: real data sets ($n\approx 500$).} To compare the algorithms' runtimes on real data, we optimize \emph{Sensor Placement} on California roadway traffic data; \emph{Movie Recommendation} on MovieLens data; \emph{Revenue Maximization} on YouTube Network data; and \emph{Influence Maximization} on Facebook Network data. See Appendix~\ref{sec:real_data} for additional details.
 \end{itemize}


\begin{figure}
\centering
\includegraphics[height=0.45\textwidth]{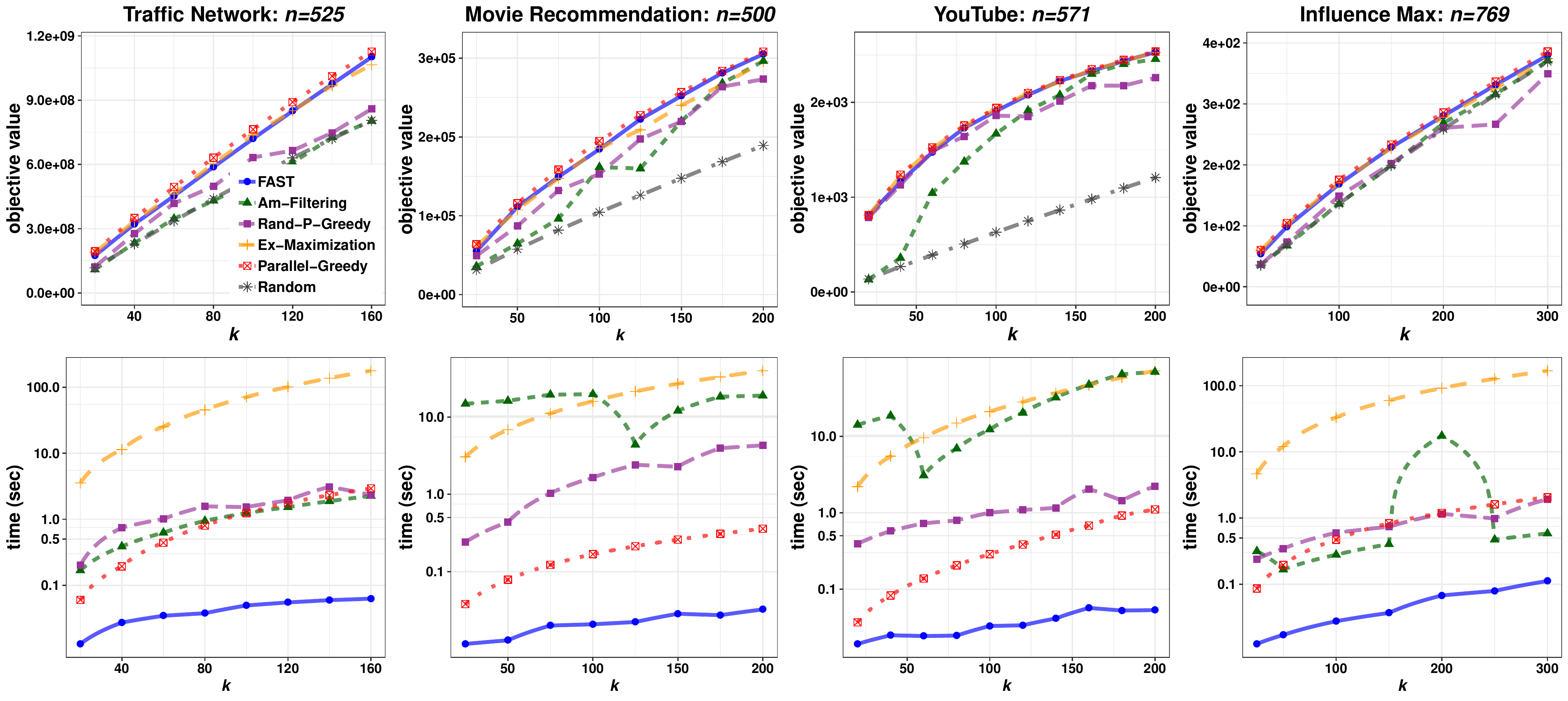}
\caption{\textit{Experiment Set 1.b: \algoptimized \ vs. low-adaptivity algorithms on real data (time axis\textbf{ log-scaled})}.}
\label{fig:adaptive_real_data}
\end{figure}

\paragraph{Results of experiment set 1.} \label{ssec:resultsLTLG} Figures \ref{fig:adaptive_rand_graphs} and \ref{fig:adaptive_real_data} plot all algorithms' solution values and parallel runtimes on synthetic and real data. In terms of solution values, across all experiments, values obtained by \algoptimized \ are nearly indistinguishable from values obtained by \textsc{Greedy}---the heuristic upper bound. From this comparison, it is clear that \algoptimized \ does not compromise on the values of its solutions.  In terms of runtime, \algoptimized \ is $18$ to $2800$ times faster than \textsc{Exhaustive-Maximization}; $3$ to $145$ times faster than \textsc{Randomized-Parallel-Greedy}; and $8$ to $1300$ times faster than \textsc{Amortized-Filtering} on the $8$ objectives and various $k$ (the time axes of Figures \ref{fig:adaptive_rand_graphs} and \ref{fig:adaptive_real_data} are log-scaled). We emphasize that \algoptimized's faster runtimes were obtained despite the fact that the three other low-adaptivity algorithms used only a single sample per processor at each of their iterations. 


\subsection{Experiment set 2: \algoptimized \ vs. Parallel-Lazier-than-Lazy-Greedy}
Our second set of experiments compares \algoptimized \ to a parallel version of \textsc{Lazier-than-Lazy-Greedy} (\textsc{LTLG})~\cite{mirzasoleiman2015lazier}, which is widely regarded as the fastest algorithm for submodular maximization in practice. Specifically, we build an optimized, scalable, truly parallel MPI implementation of \textsc{LTLG} which we refer to as \textsc{Parallel-LTLG} (see Appendix~\ref{sec:pltlg}). 

This allows us to scale up to random graphs with $n\approx100000$, large real data with $n$ up to $26000$, and various $k$ from $25$ to $25000$.  For these large experiments, running the parallel \textsc{Greedy} algorithm is impractical. \textsc{LTLG} has a $(1- 1/e - \varepsilon)$ approximation guarantee in expectation, so we likewise set both algorithms' parameters $\varepsilon$ to guarantee a $(1- 1/e - 0.1)$ approximation in expectation (see Appendix~\ref{app:epsilons}).


\begin{figure}
\centering
\includegraphics[height=0.45\textwidth]{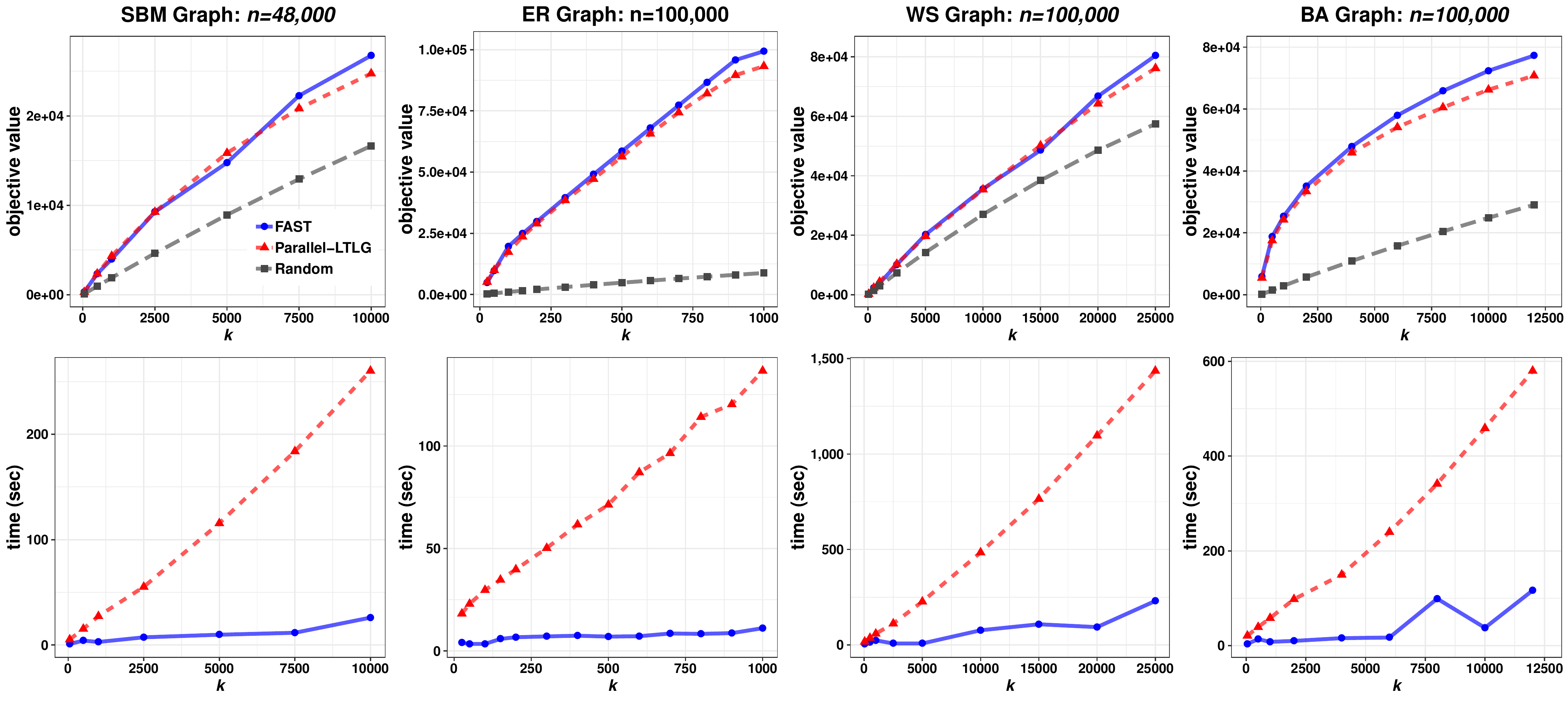}
\caption{\textit{Experiment Set 2.a:  \algoptimized \ (blue) \ vs. \textsc{Parallel-LTLG} (red) on graphs.}}
\label{fig:LTLG_rand_graphs}
\end{figure}

\begin{figure}
\centering
\includegraphics[height=0.45\textwidth]{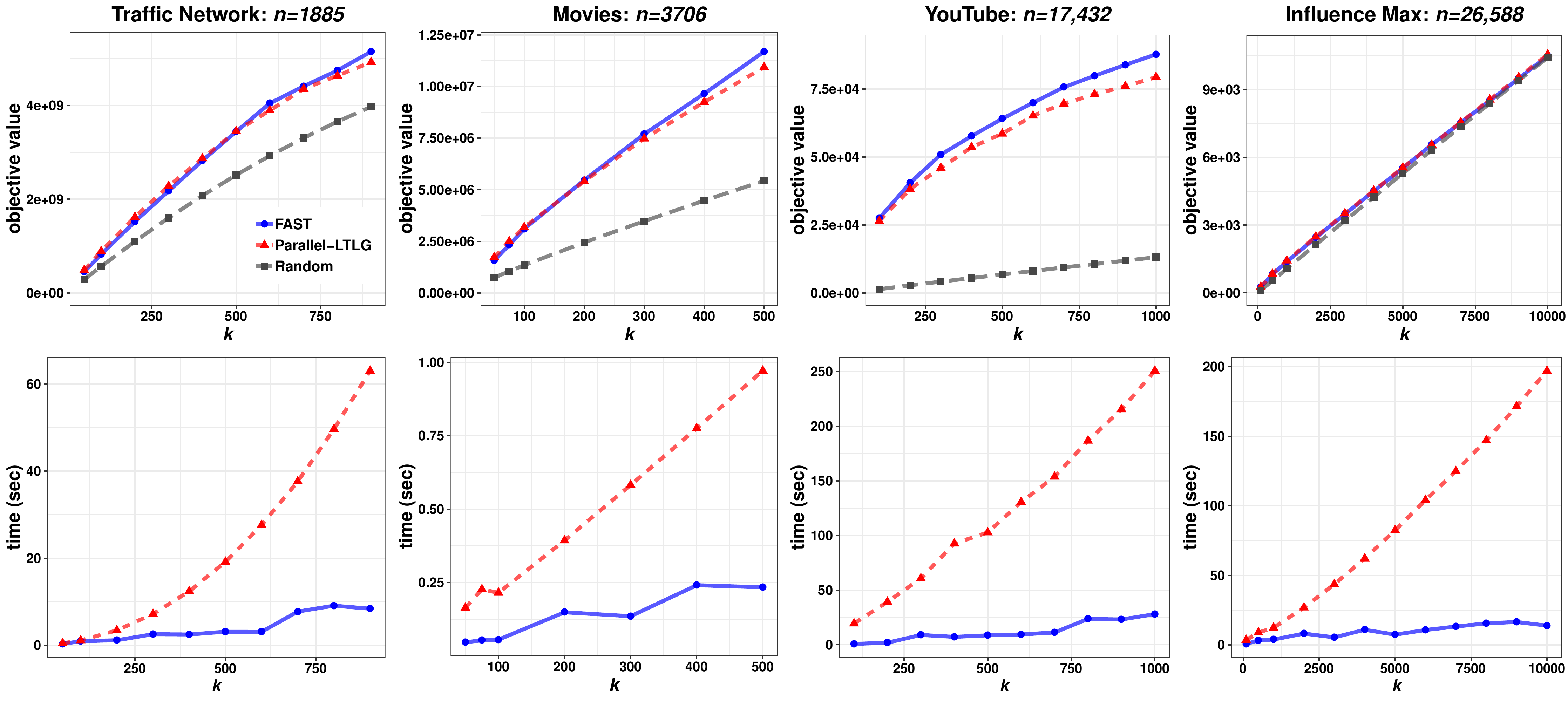}
\caption{\textit{Experiment Set 2.b:  \algoptimized \ (blue) \ vs. \textsc{Parallel-LTLG} (red) on real data.}}
\label{fig:LTLG_real_data}
\end{figure}


\label{ssec:resultsLTLGexp}
\paragraph{Results of experiment set 2.} 
 Figures \ref{fig:LTLG_rand_graphs} and \ref{fig:LTLG_real_data} plot solution values and runtimes for large experiments on synthetic and real data.  In terms of solution values, while the two algorithms achieved similar solution values across all 8 experiments, \algoptimized \ obtained slightly higher solution values than \textsc{Parallel-LTLG} on the majority of objectives and values of $k$ we tried. 
 
  In terms of runtime, \algoptimized \ was $1.5$ to $27$ times faster than \textsc{Parallel-LTLG} on each of the $8$ objectives and all $k$ we tried from $k=25$ to $25000$. More importantly, runtime disparities between \algoptimized \ and \textsc{Parallel-LTLG} increase in larger $k$, so larger problems exhibit even greater runtime advantages for \algoptimized. 
  
  Furthermore, we emphasize that due to the fact that the sample complexity of \textsc{Parallel-LTLG} is less than 95 for many experiments, it cannot achieve better runtimes by using more processors, whereas \algoptimized \ can leverage up to $n$ processors to achieve additional speedups. Therefore, \algoptimized's fast relative runtimes are a loose lower bound for what can be obtained on larger-scale hardware and problems. Figure~\ref{fig:speedup} plots \algoptimized's parallel speedups versus the number of processors we use. 
  
  Finally, we note that \emph{even when running the algorithms on a single processor}, \algoptimized \ is faster than \textsc{LTLG} for reasonable values of $k$ on $7$ of the $8$ objectives due to the fact that \algoptimized \ often uses fewer queries (see Appendix \ref{ssec:LTLGqueries}). For example, Figure~\ref{fig:speedup} plots single processor runtimes for the YouTube experiment.

\begin{figure}
  \centering
    \includegraphics[height=0.250\textwidth]{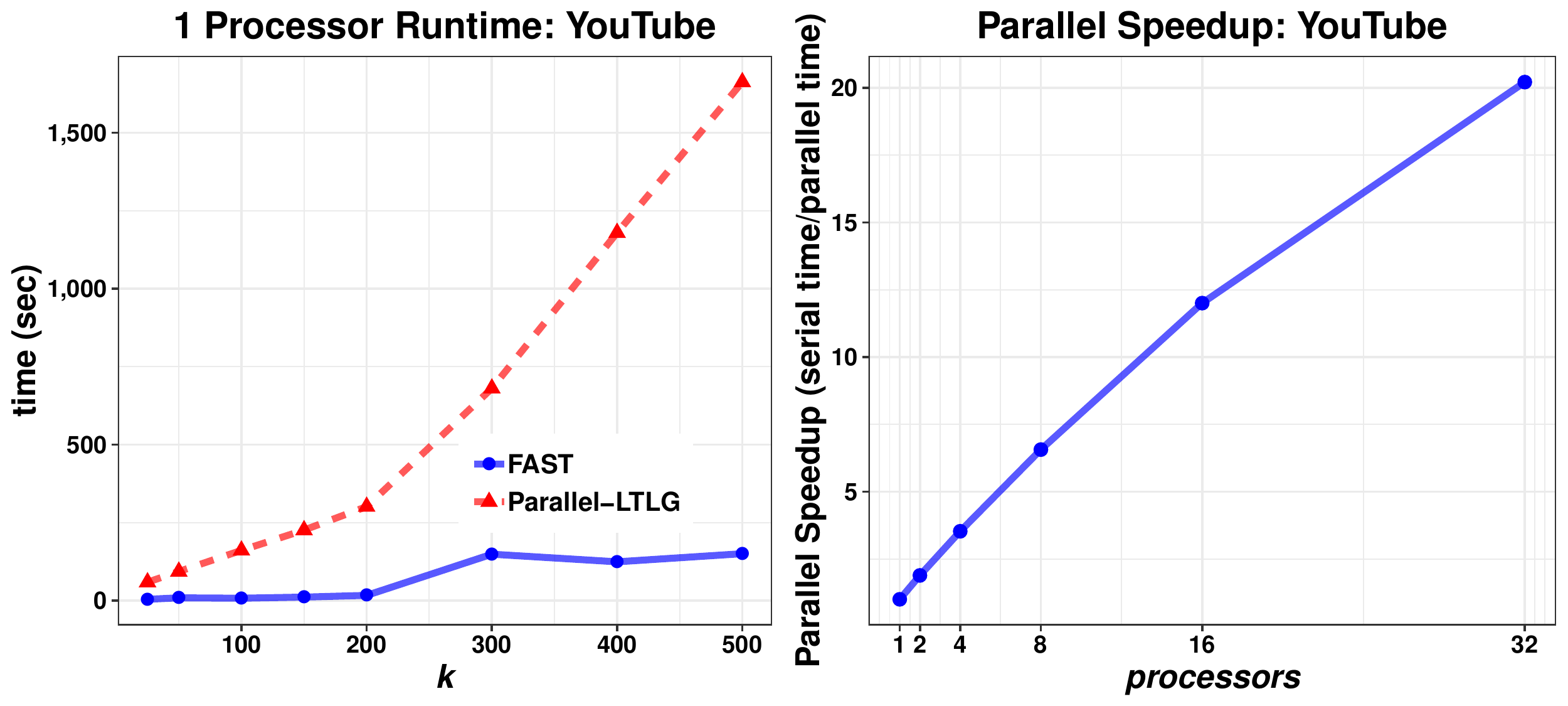}
\caption{\textit{Single processor runtimes  for \algoptimized \ and  \textsc{Parallel-LTLG},  and parallel speedups vs. number of processors for \algoptimized \ for the YouTube experiment.}}
\label{fig:speedup}
\end{figure}

\section*{Acknowledgments}
This research was supported by NSF grant 1144152, NSF grant 1647325, a Google PhD Fellowship, NSF grant CAREER CCF 1452961, NSF CCF 1816874, BSF grant 2014389, NSF USICCS proposal 1540428, a Google Research award, and a Facebook research award.

\newpage


\bibliographystyle{alpha}
\bibliography{biblio}


\newpage 

\appendix

\section*{Appendix}

\section{Comparison of Query and Adaptive Complexity with Previous Work}
\label{sec:appcomparison}

We compare the query and adaptive complexity achieved by \algoptimized \ to previous work. As mentioned in the introduction, \algoptimized \ improves the $\varepsilon$ dependency in the linear term of the query complexity of at least $\O(\varepsilon^{-5}n)$ in~\cite{BRS19, EN19} and $\O(\varepsilon^{-3}n)$ in~\cite{FMZ19} to $\O(\varepsilon^{-2} n)$.  We provide the first non-asymptotic bounds on the query and adaptive complexity  of an algorithm with sublinear adaptivity, showing dependency on small constants.  In Table~\ref{tab:queries}, we compare our (asymptotic) bounds to the query and adaptive complexity achieved in previous work.

\begin{table}[h]
\begin{center}
\begin{tabular}{lllllll}\vspace{0.2cm}
        &              & \emph{Query complexity} & \emph{Adaptivity} &   \\\ \vspace{0.3cm}
 \textsc{Amortized-Filtering} & \cite{BRS19}        &  $\mathcal{O}\left( \frac{nk^2}{ \varepsilon^{5}}  \log^2(n) \log\left(\frac{n}{\delta}\right)\right)  $  & $\mathcal{O}\left( \frac{\log n }{  \varepsilon^{2}} \right)    $  & \\\vspace{0.3cm}
 \textsc{ Exhaustive-Maximization} & \cite{FMZ19}    &  $\mathcal{O}\left(\frac{n }{ \varepsilon^{3}}+ \frac{\log\left(\delta^{-1}n\right)}{ \varepsilon^{6}} \log\left(\frac{\log n}{\delta \varepsilon^2}\right)\right)$  & $\mathcal{O}\left( \frac{\log\left(\frac{n}{\delta \varepsilon^{2} }\right)}{\varepsilon^{2}} \right) $ & \\\ \vspace{0.3cm}
 \textsc{Randomized-Parallel-Greedy}& \cite{chekuri2018submodular}& $\mathcal{O}\left(  \frac{n}{\varepsilon^4}  \log\left(\frac{n}{\delta}\right) \log^2 n\right)  $ & $ \mathcal{O}\left(\frac{\log n}{\varepsilon^2}  \right)$ & \\\ \vspace{0.3cm}
  \textsc{Fast}     & &     $\mathcal{O}\left(\frac{  n \ell}{\varepsilon^2} +  \frac{ \ell^2 \log n}{\epsilon^4}   \log(\frac{\ell\log n}{\delta \varepsilon^2})\right)$        &   $\mathcal{O}\left(\frac{ \ell^2 \log n}{\varepsilon^2}\right)$   & 
\end{tabular}
  \caption{Comparison on the query complexity and adaptivity achieved in previous work and in this paper in order to obtain a $1-1/e-\epsilon$ approximation with probability $1 - \delta$, where  $\ell = \log(\frac{\log k}{\epsilon})$.  }
  \label{tab:queries}
\end{center}
\end{table}


\section{Vanilla \algone}
\label{sec:algone}

We begin by describing a simplified version of the algorithm.  This algorithm is $\varepsilon^{-2}  \log n$ adaptive (without additional dependence on constants), uses a total of $\varepsilon^{-2} n k$ queries (again, no additional constants), and obtains a $1-\frac{1}{e} - \frac{3}{2} \varepsilon$ approximation in expectation. Importantly, it assumes the value of the optimal solution \OPT \ is known.  The full algorithm is an optimized version which does not assume \OPT \ is known and improves the query complexity.

\subsection{Description of \algone}  \algone, formally described below as Algorithm~\ref{alg:1},  generates at every iteration a  random sequence $a_1, \ldots, a_k$ of elements that is used to both add elements to the current solution and discard elements from further consideration. More precisely, each element $a_i$, for $i \in [k]$, in $\sequence(X, k)$ is a uniformly random element from the set  of surviving elements $X$, which initially contains all elements. The algorithm  identifies a position $i^{\star}$ in this sequence which determines the elements $a_1, \ldots, a_{i^{\star}-1}$ that are added to the current solution $S$, as well as the elements $a \in X$ with low contribution to $S \cup \{a_1, \ldots, a_{i^{\star}-1}\}$ that are discarded from $X$. This position $i^{\star}$ is defined to be the smallest position such that there is at least an $\varepsilon$ fraction of elements in $X$ with low contribution to $S \cup \{a_1, \ldots, a_{i^{\star}-1}\}$. By the minimality of $i^{\star}$, we simultaneously obtain that (1) the elements  added to $S$, which are the elements before position $i^{\star}$, are likely to contribute high value to the solution  and (2) at least  an $\varepsilon$ fraction of the surviving elements have low contribution to  $S \cup \{a_1, \ldots, a_{i^{\star}-1}\}$ and are discarded. 

The algorithm iterates until there are no surviving elements left in $X$. It then lowers the threshold $t$ between high and low contribution and reinitializes the surviving elements $X$ to be all elements. The algorithm lowers the threshold at most $\varepsilon^{-1}$ times and then returns the solution $S$ obtained. 

    \begin{algorithm}[H]
\caption{\algone}
\begin{algorithmic}
    	\STATE \textbf{input}  function $f$, cardinality constraint $k$, parameter $\varepsilon$, value of optimal solution $\OPT$
    	\STATE   $S \leftarrow \emptyset $
    	\STATE \textbf{while} $|S| < k$ and number of  iterations  $< \varepsilon^{-1}$ \textbf{do}
    	\STATE  \ \ \ \ \   $X \leftarrow N,  t \leftarrow (1-\varepsilon) (\OPT - f(S))/k$
	\STATE \ \ \ \ \   \textbf{while} $X \neq \emptyset$ and $|S| < k$ \textbf{do}
	\STATE \ \ \ \ \   \ \ \ \ \  $a_1, \ldots, a_k  \leftarrow \sequence(X, k)$
\STATE \ \ \ \ \   \ \ \ \ \   $i^{\star} \leftarrow \min \{i \in \{1, \ldots, k\} :  |X_i| \leq (1 - \varepsilon) |X|\}$
\STATE \ \ \ \ \   \ \ \ \ \  \ \ \ \ \   with $X_i \leftarrow \left\{a \in X : f_{S \cup \{a_1, \ldots, a_{i-1}\}}(a) \geq   t\right\}$
	\STATE \ \ \ \ \   \ \ \ \ \    $S \leftarrow S \cup \{a_1, \ldots, a_{i^{\star} - 1}\}$
	\STATE \ \ \ \ \   \ \ \ \ \  $X \leftarrow  X_{i^{\star}}$
    	\STATE \textbf{return} $S$ 
  \end{algorithmic}
  \label{alg:1}
\end{algorithm}

 
\subsection{Analysis of \algone}
\subsubsection{The adaptive complexity and query complexity}  The main observation to bound the number of iterations of the algorithm is that, by definition of $i^{\star}$, at least  an $\varepsilon$ fraction of the surviving elements in $X$ are discarded at every iteration. Since the queries at every iteration of the inner-loop can be evaluated in parallel, the adaptive complexity  is the total number of iterations of this inner-loop. For the query complexity, we note that there are  $|X|k$ function evaluations per iteration. 

\begin{restatable}{rLem}{lemadaptivity}
\label{lem:adaptivity}
The adaptive complexity of \algone \ is at most $\varepsilon^{-2} \log n$. Its query complexity is at most $\varepsilon^{-2} nk$.
\end{restatable}
\begin{proof} We first analyze the adaptive complexity and then the query complexity.

\paragraph{The adaptive complexity.} The algorithm consists of an outer-loop and an inner-loop. We first argue that at any iteration of the outer-loop,  there are at most $\varepsilon^{-1} \log n$ iterations of the inner loop. By definition of $i^{\star}$, we have that $|X_{i^{\star}}| \leq (1 - \varepsilon) |X|$. Thus there is at least an $\varepsilon$ fraction of the elements in $X$ that are discarded at every iteration. The inner-loop terminates when $|X| = 0$, which occurs at the latest at  iteration $i$  where
$(1 - \varepsilon)^i n < 1.$ This implies that there are at most $\varepsilon^{-1} \log n$ iterations of the inner loop. The function evaluations inside an iteration of the inner-loop are non-adaptive and can be performed in parallel in one round. These are the only function evaluations performed by the algorithm.\footnote{The value of $f(S)$ needed to compute $t$ can be obtained using $f_{S \cup \{a_1, \ldots,  a_{i^{\star}- 1}\}}(a_{i^{\star}})$ that was computed in the previous iteration.}  Since there are at most $\varepsilon^{-1}$ iterations of the outer-loop, there are at most  $\varepsilon^{-2} \log n$ rounds of parallel function evaluations.

\paragraph{The query complexity.} In the inner-loop, the algorithm evaluates the marginal contribution of each element $a \in X$ to $S \cup \{a_1, \ldots, a_{i}\}$ for all $i \in \{0, \ldots, k-1\}$, so a total of $k |X|$ function evaluations. Similarly as for Lemma~\ref{lem:adaptivity}, at any iteration of the outer-loop, there are at most $\varepsilon^{-1} \log n$ iterations of the inner-loop and we have $|X| \leq (1 - \varepsilon)^j n$ at iteration $j$. We conclude that the query complexity is
 $\frac{1}{\varepsilon} \sum_{j=1}^{\frac{ \log n}{\varepsilon}}k (1 - \varepsilon)^j n <  \frac{nk}{\varepsilon^2}.$
\end{proof}

\subsubsection{The approximation guarantee}   There are two cases depending on the condition which breaks the outer-loop. The main lemma for the case where there are $\varepsilon^{-1}$ iterations of the outer-loop  is that at every iteration, the elements  added to  $S$ contribute an $\varepsilon$ fraction of the remaining value $\OPT - f(S)$. 

\begin{restatable}{rLem}{lemeps}
\label{lem:eps} 
Let $S_i$ be the current solution $S$ at the start of  iteration $i$ of the outer-loop of \algone. For any $i$, if $|S_{i+1}| < k$, then $f_{S_i}(S_{i+1}) \geq \varepsilon (\OPT - f(S_i)).$
\end{restatable}
\begin{proof} 
Since  $|S_{i+1}| < k$, $X = \emptyset$ at the end of iteration $i$. This implies that  for all elements $a \in N$, $a$ is discarded from $X$ by the algorithm at some iteration   where $f_{S \cup \{a_1, \ldots, a_{i^{\star}-1}\}}(a) <  (1-\varepsilon)(\OPT- f(S_i))/k$ for some $S_i \subseteq S \subseteq S_{i+1}$. By submodularity, for any element $a \in N$, we get
$f_{S_{i+1}}(a)\leq  (1-\varepsilon)(\OPT- f(S_i))/k.$  
Next, by monotonicity and submodularity,
$\OPT - f(S_{i+1}) \leq f_{S_{i+1}}(O) \leq \sum_{o \in O} f_{S_{i+1}}(o).$ Combining the two previous inequalities, we obtain
\begin{align*}
\OPT - f(S_{i+1}) &  \leq \sum_{o \in O} f_{S_{i+1}}(o) 
  \leq \sum_{o \in O} (1-\varepsilon)(\OPT- f(S_i))/k 
 = (1-\varepsilon)(\OPT- f(S_i)).
\end{align*}
By rearranging the terms, we get the desired result.
\end{proof}

The main lemma for the case where $|S| = k$ is that the expected  contribution of each element $a_i$ added to the current solution $S$ is arbitrarily close to a $1/k$ fraction of  the remaining value $\OPT - f(S)$.

\begin{restatable}{rLem}{lemmarg}
\label{lem:marg}
At any iteration of the inner-loop of \algone, for all $i < i^{\star}$, we have
$\E_{a_i}\left[f_{S \cup \{a_1, \ldots, a_{i-1}\}}(a_i)\right] \geq  (1 - \varepsilon)^2(\OPT - f(S))/k.$
\end{restatable}
\begin{proof}
Since $a_i$ is a uniformly random element from $X$ and $|X_i| \geq (1 - \varepsilon) |X|$ for $i < i^{\star}$, we have
\begin{align*}
\EU[a_i]\left[f_{S \cup \{a_1, \ldots, a_{i-1}\}}(a_i)\right]   \geq \Pr_{a_i}\left[f_{S \cup \{a_1, \ldots, a_{i-1}\}}(a_i) \geq t \right] \cdot t \geq (1 - \varepsilon) \cdot (1 - \varepsilon)(\OPT - f(S))/k. \hspace{0.5cm}\qedhere
\end{align*}
\end{proof}

By standard greedy analysis, Lemmas~\ref{lem:eps} and \ref{lem:marg} imply that the algorithm obtains a $1-1/e-\O(\varepsilon)$ approximation in each case. We emphasize the low constants and dependencies on $\varepsilon$ in this result compared to previous results in the adaptive complexity model.
\begin{theorem}
\label{thm:vanilla}
\algone \ is an algorithm with at most $\varepsilon^{-2}  \log n$ adaptive rounds and  $\varepsilon^{-2} n k$ queries that achieves a $1 - 1/e - \frac{3\varepsilon}{2}\varepsilon$ approximation in expectation.
\end{theorem}
\begin{proof}
We first consider the case where there are $\varepsilon^{-1}$ iterations of the outer-loop.
Let $S_1, \ldots, S_{\varepsilon^{-1}}$ be the set $S$ at each of the $\varepsilon^{-1}$ iterations of \algone.
The algorithm increases the value of the solution $S$ by at least $ \varepsilon \left(\OPT - f(S)\right)$ at every iteration by Lemma~\ref{lem:eps}. Thus,
$$f(S_i) \geq f(S_{i-1}) + \varepsilon \left(\OPT - f(S_{i-1})\right).$$
Next, we show by induction on $i$ that
$$f(S_i) \geq \left( 1 - \left(1 - \varepsilon\right)^i \right) \OPT.$$
Observe that
\begin{align*}
f(S_i) & \geq f(S_{i-1}) + \varepsilon \left(\OPT - f(S_{i-1})\right) \\
& = \varepsilon  \OPT + \left(1- \varepsilon \right) f(S_{i-1})\\
& \geq \varepsilon\OPT + \left(1- \varepsilon\right)\left(1 -  \left(1 - \varepsilon \right)^{i-1}\right) \OPT \\
& = \left( 1 - \left(1 - \varepsilon\right)^i\right) \OPT
\end{align*}
Since $1 - x \leq e^{-x}$, we get 
$$f(S_{\varepsilon^{-1}}) \geq \left( 1 -  e^{-1}\right) \OPT .$$
Similarly, for the case where the solution $S$ returned is such that $|S| = k$, by Lemma~\ref{lem:marg} and by induction we get that $$f(S) \geq \left( 1 -  e^{-(1 - \varepsilon)^2}\right) \OPT \geq \left( 1 -  e^{-(1 - 2\varepsilon)}\right) \OPT\geq \left( 1 -  e^{-1}(1 - 4\varepsilon)\right) \OPT \geq (1- e^{-1} - \frac{3}{2} \varepsilon)\OPT.$$
\end{proof}

\section{Analysis of the Main Algorithm}
\label{sec:appcardinality}

We define $\ell =  \log(\varepsilon^{-1} \log k)$ and $m = \frac{2 + \varepsilon}{\varepsilon^2(1 -  3\varepsilon)} \log(4\ell\varepsilon^{-2}\log(n)\delta^{-1})$

\subsection{Adaptive Complexity and Query Complexity} 

The adaptivity of the main algorithm is slightly worse than for \algone \ due to the binary searches over $V$ and $I$. To obtain the adaptive complexity with probability $1$, if at any iteration of the outer while-loop there are at least $\varepsilon^{-1} \log n$ iterations of the inner-loop, we declare failure. In Lemma~\ref{lem:emptyset}, we show this happens with low probability.

\begin{restatable}{rLem}{lemadaptivitytwo}
\label{lem:adaptivity2}
The adaptive complexity of \algoptimized \ is at most $\varepsilon^{-2} \log(n) \cdot \log^2(\varepsilon^{-1} \log k).$ 
\end{restatable}
\begin{proof}    The algorithm consists of four nested loops: a binary search over $V$, an outer while-loop, an inner while-loop, and a binary search over $I$. For the binary searches, we have $|V| \leq \varepsilon^{-1} \log k$ and $|I| \leq \varepsilon^{-1} \log k$. Thus, there are at most $\ell$  iterations for each binary search.
 
Due to the termination condition of the while-loops, there are at most $\varepsilon^{-1}$ and $\varepsilon^{-1} \log n$ iterations of each while-loop. The function evaluations inside an iteration of the last nested loop are non-adaptive and can be performed in parallel in one round. Thus the adaptive complexity of \algoptimized \ is a most 
$$\varepsilon^{-2} \log(n) \cdot \log^2(\varepsilon^{-1} \log k).$$
\end{proof}

Thanks to the binary search over $I$ and  the subsampling of $R$ from $X$, the query complexity is improved from $\O(\varepsilon^{-2} nk)$ to  $\tilde{O}(\varepsilon^{-2} n +\varepsilon^{-4} \log(n)  \log(\delta^{-1}))$

\begin{restatable}{rLem}{lemadaptivitytwo}
\label{lem:query2}
The query complexity of \algoptimized \ is at most  
$2\varepsilon^{-2}\ell n +  \ell^2 \varepsilon^{-2} \log(n)  m.$
\end{restatable}
\begin{proof} There are $n$ queries, $f(a)$ for all $a$, needed to compute $V$. At each iteration of the binary search over $I$, there are $m  $ queries needed for $R_i$  to evaluate $f_{S \cup \{a_1, \ldots, a_{i-1}\}}(a)$ for $a \in R$. There are at most $\ell \varepsilon^{-2}  \log n$ instances of the binary search over $I$, each with at most $\ell$ iterations. The total number of queries for this binary search is at most 
$$\ell^2 \varepsilon^{-2} \log(n)  m.$$
At each iteration $i$ of the inner-while loop, there are at most $|X| \leq (1 - \varepsilon)^i n$ queries to update $X$ and  at most $|X|$ queries to add elements $a_i$ to $S$. There are at most $\ell \varepsilon^{-1}$ instances of the inner while-loop each with at most $\varepsilon^{-1} \log n$ iterations. The total number of queries for updating $X$ and $S$ is $$\ell \varepsilon^{-1} \sum_{i=1}^{\varepsilon^{-1} \log n} 2 (1-\varepsilon)^i n \leq 2\varepsilon^{-2}\ell n.$$ By combining the queries needed to compute $V$, $R_i$, $X$ and $S$, we get the desired bound on the query complexity.
\end{proof}

\subsection{The Approximation}

\subsubsection{Finding $i^{\star}$}
Similarly as for \algone, we denote $X_{i} = \{a \in X: f_{S \cup \{a_1, \ldots, a_{i-1}\}}(a) \geq t\}$ and  $R_{i} = \{a \in R : f_{S \cup \{a_1, \ldots, a_{i-1}\}}(a) \geq t\}$.

\begin{lemma}
\label{lem:concentration}
Assume that $m = \frac{2 + \varepsilon}{\varepsilon^2(1 -  3\varepsilon)} \log(4\ell\varepsilon^{-2}\log(n)\delta^{-1})$, then, with probability $1 - \delta/2$, for all iterations of the inner while-loop, we have that
$|X_{(1 - \varepsilon)i^{\star} }| \geq (1 - 3\varepsilon)|X|$
 and $|X_{i^{\star}}| \leq (1 - \varepsilon)|X|$.
\end{lemma}
\begin{proof} By the definition of $i^{\star}$ and $I$, we have that $|R_{(1 - \varepsilon)i^{\star} }| \geq (1 - 2\varepsilon)|R|$
 and $|R_{i^{\star}}| \leq (1 - 2\varepsilon)|R|$. We show by contrapositive that, with probability $1 - \delta/4$ if $|X_{(1 - \varepsilon)i^{\star}}| \leq (1 - 3\varepsilon)|X|$ then $|R_{(1 - \varepsilon)i^{\star}}| \leq (1 - 2\varepsilon)|R|$ and that if $|X_{i^{\star}}| \geq (1 - \varepsilon)|X|$ then $|R_{i^{\star}}| \geq (1 - 2\varepsilon)|R|$.

 Note that for all $a \in R$ and $i \in [k]$, we have
$$\Pr_{a}\left[f_{S \cup \{a_1, \ldots, a_{i-1}\}}(a) \geq   t\right] = \frac{|X_{i}|}{|X|}.$$
 First, assume that $|X_{i^{\star}}| > (1 - \varepsilon)|X|$. Then  by the Chernoff bound, with $\mu = m \cdot \frac{|X_{i^{\star}}|}{|X|} \geq (1 - \varepsilon)m$,
\begin{align*}
\Pr\left[|R_{i^{\star}}| \leq (1 - 2\varepsilon)|R|\right]  \leq  \Pr\left[|R_{i^{\star}}| \leq (1 - \varepsilon)^2 m \right] \leq  e^{-\varepsilon^2 (1 - \varepsilon)m/(2 + \varepsilon)} \leq \frac{\delta}{4\ell\varepsilon^{-2}\log n}.
\end{align*}
Next, assume that $|X_{i^{\star}}| < (1 - 3\varepsilon)|X|$. By the Chernoff bound with $\mu \leq (1 - 3 \varepsilon)m$,
\begin{align*}
\Pr\left[|R_{i^{\star}}| \geq (1 - 2\varepsilon)|R|\right]  \leq  \Pr\left[|R_{i^{\star}}| \geq (1 + \varepsilon)(1 -  3\varepsilon)m\right] \leq  \frac{\delta}{4\ell\varepsilon^{-2}\log n}.
\end{align*}
Thus, with $m = \frac{2 + \varepsilon}{\varepsilon^2(1 -  3\varepsilon)} \log(4 \ell\varepsilon^{-2}\log(n)\delta^{-1})$ and by contrapositive, we have that $|X_{(1 - \varepsilon)i^{\star}}| \geq (1 - 3\varepsilon)|X|$
 and $|X_{i^{\star}}| \leq (1 - \varepsilon)|X|$ each with probability $1 - \delta/(4\ell\varepsilon^{-2}\log n) $. By a union bound, these both hold with probability $1 - \delta/2$ for all $\ell\varepsilon^{-2}\log n$ iterations of the inner while-loop.
\end{proof}

\begin{corollary}
\label{lem:discarding}
With probability $1 - \delta/2$, for all iterations of the inner while-loop,  we have
\begin{itemize}
\item $|X_i| \geq (1 - 3 \varepsilon)|X|$ for all $i < (1 - \varepsilon)i^{\star}$, and
\item $|X_i| \leq (1 - \varepsilon) |X|$ for all $i \geq i^{\star}$
\end{itemize}
\end{corollary}
\begin{proof}
Consider an iteration of the inner while-loop. We first note that, by submodularity,  $|R_i|$ is monotonically decreasing as $i$ increases. Thus we can perform a binary search over $I$ to find $i^{\star}$. By Lemma~\ref{lem:concentration}, we have that  with probability $1 - \delta/2$, we have that
$|X_{(1 - \varepsilon)i^{\star}}| \geq (1 - 3\varepsilon)|X|$
 and $|X_{i^{\star}}| \leq (1 - \varepsilon)|X|$. We conclude the proof by noting that by submodularity, $|X_i|$ is also monotonically decreasing as $i$ increases.
\end{proof}

\begin{lemma}
\label{lem:emptyset}
With probability $1 - \delta/2$, at every iteration of the outer while-loop, $X = \emptyset$ after at most $ \varepsilon^{-1} \log n$ iterations of the inner while-loop.
\end{lemma}
\begin{proof}
By Lemma~\ref{lem:discarding}, with probability $1 - \delta/2$, at every iteration of the inner while-loop,  there is at least an $\varepsilon$ fraction of the elements in $X$ that are discarded. We assume this is the case. After $\varepsilon^{-1} \log n$ iterations of discarding an $\varepsilon$ fraction of the elements in $X$, we have $X = \emptyset$.

\end{proof}

\subsubsection{If number of iterations of outer while-loop is $\varepsilon^{-1}$}

The analysis defers depending on whether the number of iterations or the size of the solution caused the algorithm to terminate. We first analyze the case where \algone \ returned $S$ s.t. $|S| < k$ because the number of iterations reached $\varepsilon^{-1}$. The main lemma for this case is that at every iteration of \algone, if $v \leq \OPT$, the set $T$ added to the current solution $S$ contributes at least an $\varepsilon$ fraction of the remaining value $v - f(S)$. 

\begin{lemma}
\label{lem:eps2} Assume that $v \leq \OPT$ and let $S_i$ be the set $S$ at the start of iteration $i$ of the outer while-loop  of \algoptimized. With probability $1 - \delta/2$, for all $v \in V$ and all $i \leq \varepsilon^{-1}$, we have that  if $|S_{i+1}| < k$, then $f_{S_i}(S_{i+1}) \geq \varepsilon (v - f(S_i)).$
\end{lemma}
\begin{proof} By Lemma~\ref{lem:emptyset}, with probability $1 - \delta/2$, at every iteration of the outer while-loop, $X = \emptyset$ after at most $\varepsilon^{-1} \log n$ iterations of the inner while-loop. We assume this holds for the remaining of this proof.

Since  $|S_{i+1}| < k$, $X = \emptyset$ at the end of iteration $i$ of the outer while-loop. This implies that  for all elements $a \in N$, $a$ is discarded from $X$ by the algorithm at some iteration   where $$f_{S \cup \{a_1, \ldots, a_{i^{\star}-1}\}}(a) <  (1-\varepsilon)(v - f(S_i))/k$$ for some $S_i \subseteq S \subseteq S_{i+1}$. By submodularity, for any element $a \in N$, we get
$$f_{S_{i+1}}(a)\leq  (1-\varepsilon)(v - f(S_i))/k.$$  
Next, since $v \leq \OPT$, by monotonicity, and by submodularity,
$$v - f(S_{i+1})  \leq \OPT - f(S_{i+1}) \leq f_{S_{i+1}}(O) \leq \sum_{o \in O} f_{S_{i+1}}(o).$$ Combining the previous inequalities, we obtain
\begin{align*}
v - f(S_{i+1}) &  \leq \sum_{o \in O} f_{S_{i+1}}(o) 
  \leq \sum_{o \in O} (1-\varepsilon)(v - f(S_i))/k 
 = (1-\varepsilon)(v - f(S_i)).
\end{align*}

By rearranging the terms, we get the desired result.
\end{proof}

By standard greedy analysis, similarly as for the proof of Theorem~\ref{thm:vanilla} we obtain that $f(S) \geq (1- 1/e) v$.

\begin{lemma}
\label{lem:iterationsapx2} With probability $1 - \delta/2$, for all $v \leq \OPT$, after $\varepsilon^{-1}$ iterations of the outer while-loop of \algoptimized, $f(S) \geq (1-1/e) v$.
\end{lemma}

\subsubsection{If $|S| = k$} Next, we analyze the case where the outer-loop terminated because $|S| = k$. We show that each element added to $S$  is, in expectation, a good approximation to $t_1$.

\begin{lemma}
\label{lem:marg2} With probability $1 - \delta/2$,
at every iteration of the inner while-loop, we have that independently for each $i \leq (1 - \varepsilon)i^{\star}$, with probability at least $1 - 3 \varepsilon$,
$$f_{S \cup \{a_1, \ldots, a_{i-1}\}}(a_i) \geq  (1 - \varepsilon)(v - f(S))/k.$$
\end{lemma}
\begin{proof}
By Corollary~\ref{lem:discarding}, we have that with probability $1 - \delta/2$, for all iterations of the inner while-loop,  $$\left|\{a \in X: f_{S \cup \{a_1, \ldots, a_{i-1}\}}(a) \geq  (1 - \varepsilon)(v - f(S))/k\}\right| = |X_i| \geq (1 - 3\varepsilon) |X|$$ for all $i \leq (1 - \varepsilon)i^{\star}$. We assume this is the case  and consider an iteration of the inner while-loop.  Since each $a_i$ is a uniformly random element from $X$, we have that independently for each $i \leq (1 - \varepsilon)i^{\star}$, 
$$\Pr_{a_i}\left[f_{S \cup \{a_1, \ldots, a_{i-1}\}}(a_i) \geq  (1 - \varepsilon)(v - f(S))/k \right] \geq 1 - 3 \varepsilon.$$
\end{proof}

\subsubsection{Guessing $\OPT$}

\begin{lemma}[Extends \cite{HS17}]
\label{lem:HS17}
Consider a set $S = \{a_1, \ldots, a_{|S|}\}$  and let $S_i = \{a_1, \ldots, a_i\}$. Assume that, independently for each $i \in [|S|]$,   we have that with probability at least $1 - \delta$,
$$f_{S_{i-1}}(a_i) \geq \mu \cdot \frac{1}{k}(v - f(S_{i-1})),$$
then, for any $\varepsilon \in (0, 1)$ such that $|S| \geq \frac{1}{\varepsilon^2 (1-\delta)\mu}$, 
$$f(S) \geq \left(1 - e^{- \frac{|S|}{k}(1 - \delta)\mu(1 - \varepsilon)}\right) v$$
with probability at least  $ 1 -  e^{- |S|(1-\delta)\mu   \varepsilon^2/2}$.
\end{lemma}
\begin{proof}
The analysis is similar as in  \cite{HS17}. Assume that $f_{S_{i-1}}(a_i) = \xi_i \cdot \frac{1}{k}(v - f(S_{i-1}))$ and let $\hat{\mu} = \frac{1}{k}\sum_{i=1}^{|S|} \xi_i$. We first argue that  
$f(S) \geq \left(1 - e^{- \hat{\mu}}\right) v$. By induction, we have that 
$$f(S_i) \geq \left(1 - \prod_{j=1}^i\left(1 - \frac{\xi_j}{k}\right)\right) v.$$
Since $1 - x \leq e^{-x}$, we obtain
\begin{align*}
f(S) \geq  \left(1 - \prod_{j=1}^{|S|}\left(1 - \frac{\xi_j}{k}\right)\right) v \geq \left(1 - e^{-\sum_{j=1}^{|S|} \frac{\xi_j}{k}}\right)v\left(1 - e^{-\hat{\mu}}\right)v.
\end{align*}
Let $S' = \{a_i \in S:  f_{S_{i-1}}(a_i) \geq \mu \cdot \frac{1}{k}(v - f(S_{i-1}))$. By the Chernoff bound, 
$$\Pr\left[|S'| < (1 - \varepsilon) (1 - \delta)|S|\right] \leq e^{-\frac{\varepsilon^2(1 - \delta)|S|}{2}}.$$
Thus, with probability at least $1 - e^{-\frac{\varepsilon^2(1 - \delta)|S|}{2}}$, we get
 $$\hat{\mu} \geq \frac{1}{k} |S'| \mu \geq \frac{|S|}{k} (1 - \varepsilon) (1 - \delta) \mu.$$
\end{proof}

\begin{lemma}
\label{lem:main} Assume $k \geq  \frac{2 \log(2\delta^{-1} \ell)}{\varepsilon^2 (1 - 5\varepsilon)}$.
With probability at least $1  -   \delta$, we have that for all $v$  at some iteration of the binary search over $V$ such that $v \leq \OPT$, 
$$f(S) \geq \left(1 - e^{-(1 - 6\varepsilon)}\right) v.$$
\end{lemma}
\begin{proof} With probability $1- \delta/2$,  Corollary~\ref{lem:discarding}, and consequently Lemma~\ref{lem:iterationsapx2} and Lemma~\ref{lem:marg2}, hold for all iterations of the inner while-loop and we assume this is the case for the remainder of this proof.

Consider $v$ at some  iteration of the binary search over $V$.
If the outer while-loop terminated after $\varepsilon^{-1}$ iterations, then by Lemma~\ref{lem:iterationsapx2}, we have $f(S) \geq \left(1 - 1/e\right) v$.

Otherwise, the outer while-loop terminated with $|S| = k$. The algorithm adds elements to $S$ that are of two types: those added before the if condition and those in $A_{i^{\star}}$ added after. Let $T \subseteq S$ be the set obtained by discarding from $S$ the elements  $a \in A_{i^{\star}}$ that, at the iteration of the inner while-loop where $a$ was added, had position $i$ in the sequence $a_1, \ldots, a_{k - |S|}$ such that $(1 - \varepsilon)i^{\star} \leq i \leq i^{\star}$. This set $T$ is such that  $|T| \geq (1 - \varepsilon)|S| = (1 - \varepsilon)k$.

 Consider $a_i \in S$ added before the if condition and  let $S_{i-1}$ be the set of elements in $S$ added to $S$ before $a_i$. Since $a_i$ was added to $S$, by submodularity,  we have $f_{S_{i-1}}(a_i) \geq (1 - \varepsilon) \frac{1}{k}(v - f_{S_{i-1}})$ with probability $1$. Consider an elements $a_i \in T$.  By Lemma~\ref{lem:marg2} and by definition of $T$, we have that  independently for each $a_i \in T$, with probability at least $1 - 3 \varepsilon$, 
$f_{S_{i-1}}(a_i) \geq (1 - \varepsilon) \frac{1}{k}(v - f_{S_{i-1}})$ where 
$S_{i-1}$ is the set of elements in $T$ added to $T$ before $a_i$.

By Lemma~\ref{lem:HS17} and since $|T| \geq (1 - \varepsilon) k$, with $\delta = 3 \varepsilon$, $\mu = 1 - \varepsilon$, and $\varepsilon = \varepsilon$, we have that $f(T) \geq \left(1 - e^{-(1 - \varepsilon)(1 - \varepsilon)(1 - 3 \varepsilon)(1- \varepsilon)}\right) v \geq \geq \left(1 - e^{-(1 - 6\varepsilon)}\right) v$ with probability at least $1 - \delta/(2\ell)$ if $k \geq \frac{2 \log(2\delta^{-1} \ell)}{\varepsilon^2 (1 - \varepsilon)(1 - \varepsilon) (1 - 3\varepsilon)} \geq \frac{2 \log(2\delta^{-1} \ell)}{\varepsilon^2 (1 - 5\varepsilon)} $. By monotonicity, $f(S) \geq f(T)$.

By a union bound over all $\ell$ iterations of the binary search over $V$, for all $v$ considered during this binary search, we have that $f(S) \geq \left(1 - e^{-(1 - 6\varepsilon)}\right) v$ with probability at least $1 - \delta/2$. 
\end{proof}

\thmmain*
\begin{proof}
By the definition of $V$ and since $\max_{a \in N} f(a) \leq \OPT \leq \max_{|S| \leq k} \sum_{a \in S} f(a)$, there exists $v' \in V$ such that $v' \in [(1 - \varepsilon) \OPT, \OPT]$. 

By Lemma~\ref{lem:main}, with probability at least $1 -  \delta$, we have that for all $v$  at some iteration of the binary search over $V$ such that $v \leq \OPT$, 
$f(S) \geq \left(1 - e^{-(1 - 6\varepsilon)}\right) v.$ Since $v' \leq \OPT$, it must be the case that $v^{\star} \geq v'$ and we get
$$\left(1 - e^{-(1 - 6\varepsilon)}\right)^{-1}f(S_{v^{\star}}) \geq \ v^{\star} \geq v' \geq  (1 - \varepsilon) \OPT.$$ 

For $\varepsilon \in (0, 0.1)$, we have $e^{6\varepsilon} \leq 1 + 9 \varepsilon$. We get
$$(1 - \varepsilon)(1-e^{-(1 - 6\varepsilon}) \geq (1 - \varepsilon)(1-e^{-1}(1 + 9\varepsilon)) \geq 1 - e^{-1} - 4 \varepsilon.$$
\end{proof}



\section{Additional Information for Experiments}
\subsection{Benchmark algorithms}
\label{ssec:experiments_algos}
Our first set of experiments compares \algoptimized's performance to three state-of-the-art low-adaptivity algorithms:

\begin{itemize}
\item \textbf{Amortized-Filtering \citep{BRS19}.} At each round, \textsc{Amortized-Filtering} sets an adaptive value threshold based on the value of its current solution. It uses this threshold to filter remaining elements into high-value and low-value groups. It then adds a randomly chosen set of high-value elements to the solution and updates the threshold for the next round. \textsc{Amortized-Filtering} achieves a $(1-1/e-\epsilon)$ approximation in $O(\log (n)\varepsilon^{-3})$ rounds. When \OPT \ is unknown as it is in all experiments here, \textsc{Amortized-Filtering} is typically run once each for several geometrically increasing guesses for \OPT \ to maintain this approximation guarantee.
\item \textbf{Randomized-Parallel-Greedy \citep{chekuri2018submodular}.}  At each round, \textsc{Randomized-Parallel-Greedy} partitions elements into high-value and low-value groups based on their marginal contributions. It then uses the multilinear extension to estimate the maximum probability (`step size') with which it can randomly add elements from the high value group to the solution while maintaining theoretical guarantees. This algorithm achieves a $(1-1/e- \epsilon)$-approximation in $O(\log(n)\varepsilon^{-2})$ parallel rounds. When \OPT \ is unknown, \textsc{Randomized-Parallel-Greedy} must be run with at least one upper-bound guess for \OPT \ to maintain this approximation guarantee.
\item \textbf{Exhaustive-Maximization \citep{FMZ19}.} \textsc{Exhaustive-Maximization} begins by fixing a value threshold. It then iteratively partitions remaining elements into high-value and low-value groups by determining whether each element's average marginal contribution to a random set exceeds the value threshold. It draws elements from the high-value group to form a candidate solution $S$. Finally, it lowers the value threshold and repeats the process for several rounds, keeping track of the candidate solution with the highest value. \textsc{Exhaustive-Maximization} achieves a $(1-1/e- \epsilon)$-approximation in $O(\log (n)\varepsilon^{-2})$ adaptive rounds.
\end{itemize}

Our second set of experiments compares \algoptimized \ to a parallel version of \textsc{Lazier-than-Lazy-Greedy} (\textsc{LTLG}):
\begin{itemize}
\item \textbf{Parallel-Lazier-than-Lazy-Greedy (Parallel-LTLG) ~\cite{mirzasoleiman2015lazier}.} \textsc{LTLG} is widely regarded as the fastest algorithm for submodular maximization in practice. At each round, it draws a small random sample of elements. It then attempts a lazy update via a single query by testing whether the element in the sample with the highest previously-computed marginal value has a current marginal value that exceeds the second-highest previously-computed marginal value among the samples. If this is the case, it adds this best element to the solution. Otherwise, it computes marginal contribution of all samples and adds the best element in the sample set to the solution. It achieves a $(1-1/e- \epsilon)$ approximation in $k$ adaptive rounds.
\end{itemize}

For calibration, we also ran (1) a parallel version of the standard \textsc{Greedy} algorithm and (2) an algorithm that estimates the value of a random solution:

\begin{itemize}
\item \textbf{Parallel-Greedy.} \textsc{Greedy} iteratively adds the element with the highest marginal value to the solution set $S$ for each of $k$ rounds. \textsc{Greedy} achieves a $1-1/e$ approximation in $k$ adaptive rounds, and its solution values are widely regarded as an heuristic upper bound.


\item \textbf{Random.} \textsc{Random} returns the average value of a randomly chosen set $S$ of $k$ elements.
\end{itemize}

\subsection{Choosing parameters $\varepsilon$ and $\delta$}
\label{app:epsilons}
For \emph{Experiment set 1}, we choose all algorithms' $\delta$ and $\varepsilon$ such that each guarantees a $(1-1/e-0.1)$ approximation with probability $0.95$. We therefore choose $\delta=0.95$ for all algorithms and set $\varepsilon$ to $0.025$ for \algoptimized; $0.1$ for \textsc{Amortized-Filtering}; $0.1$ for for \textsc{Exhaustive-Maximization}; and $0.048$ for \textsc{Randomized-Parallel-Greedy}. For \emph{Experiment set 2}, the $(1-1/e-\varepsilon)$ approximation guarantee of \textsc{LTLG} holds in expectation, so we set $\varepsilon=0.1$ for \textsc{Parallel-LTLG} and $\varepsilon=0.025$ for \algoptimized, which gives the same $(1-1/e-0.1)$ approximation in expectation (see Theorem \ref{thm:main}).

\subsection{Objective functions and data sets}
\label{ssec:datasets}

\subsubsection{Max cover on random graphs}\label{ssec:random_graphs}
Recall the max cover objective: given a graph $G$, the cover function $f(S)$ measures the count of nodes with at least one neighbor in $S$. This is a canonical monotone submodular function. To compare the algorithms' runtimes under a range of conditions, we solve max cover on synthetic graphs generated via four different well-studied graph models:

\begin{itemize}
\item \textbf{Erd\H{o}s R\'{e}nyi.} We generate $G(n, p)$ graphs with a $p=0.01$ probability of each edge. Since many nodes have similar degree in this model and each node's edges are spread randomly across the graph, a random set of nodes often achieves good coverage.
\item \textbf{Stochastic block model.} We generate SBM graphs with a $p=0.1$ probability of an edge between each pair of nodes in the same cluster. Here, we expect that a good solution will cover nodes in all clusters. 
\item \textbf{Watts-Strogatz.} We generate WS graphs initialized as ring lattices with $2$ edges per node and a $p=0.1$ probability of rewiring edges. In these `small-world' graphs, many nodes have identical degree, so good solutions contain nodes chosen to minimize  coverage overlaps. 
\item \textbf{Barb\'{a}si-Albert.} We generate BA graphs with $m=1$ edges added per iteration. BA graphs exhibit scale-free structure and tend to have a small set of high-degree nodes. Therefore, it is often possible to obtain high coverage in these graphs by choosing the  highest degree nodes.
\end{itemize}

\subsubsection{Random graphs: experiment size} 
For ER, WS, and BA graphs, we set $n=500$ in our small experiments and $n=100,000$ in our large experiments. For SBM graphs, we fix parameters to approximately match these sizes in expectation, as the actual size of an SBM graph is a draw from a random process. Specifically, for small SBM experiments we draw $10$ clusters of $10$ to $100$ nodes each, and for large experiments we draw $50$ clusters of $100$ to $5000$ nodes each.

\subsubsection{Real data}\label{sec:real_data}
\begin{itemize}
%
%
%
\item \textbf{Traffic speeding sensor placement.} In this application, we select a set of locations to install traffic speeding sensors on a highway network, and our objective is to choose locations to maximize the traffic that the sensors observe. Similarly to ~\citep{BBS18}, we conduct this experiment using data from the CalTrans PeMS system ~\citep{pems}, which allows us to reconstruct the directed network where nodes are locations on each California highway ($40,000$ locations) and directed edges are the total count of vehicles that passed between adjacent locations in April, 2018. We use the directed, weighted max cover function to measure the total count of traffic observed at a set of sensor locations. For a given set $S$ of sensor locations, this objective function returns the sum of edge weights (traffic counts along roadway sections) for which at least one endpoint is in $S$. For our small experiments, we follow \citep{BBS18} and restrict the network to the $521$ locations within a $10$ miles of the Los Angeles city center. For our large experiments, we expand this to all of the $\sim 2000$ locations in the region.
%
%

\item \textbf{Movie recommendation.} In movie recommendation, the objective is to recommend a small, diverse, and highly-rated set of movies based on a data set of users' movie ratings. We use the objective function and dataset from \citep{BS18b}, which sums the ratings of movies in the set $S$ and includes a diversity term that captures how well the chosen set of movies covers the set of movie genres in the data. A good set of movies includes movies that have high overall ratings, but it also should appeal to users' different tastes by including at least one film that each user rates very highly. Therefore, we also include a diversity term that counts the number of users who would give a high rating to at least one film in the set of movies $S$. We obtain the following objective:

\begin{align}
\label{eqn:movie_objective}
    f(S) = \sum_{i\in U} \sum_{j\in S} r_{i,j} +\alpha C(S) + \beta D(S) 
\end{align}

where $U$ is the set of users $i$, $r_{i,j}$ is user $i$'s predicted rating of movie $j$; $C(S)$ is a coverage function that counts the number of different genres covered by $S$; $D(S)$ is a coverage function that counts the number of users with at least one highly rated film in $S$; and parameters $\alpha\ge 0$ and $\beta\ge 0$ control the relative weight that the objective function places on highly rated movies versus diversity. Note that eqn. \ref{eqn:movie_objective} is a monotone submodular function. As in \citep{BS18b}, we predict missing ratings for the user-movie ratings matrix using the standard approach of low-rank matrix completion via the iterative low-rank SVD decomposition algorithm \textsc{SVDimpute} analyzed in ~\citep{troyanskaya2001missing}. We set $\alpha=0.5 \max_j(\sum_i r_j)$ and $\beta=1$, and we define a high rating as $r_j>4.5$ (which corresponds to $1$\% of the ratings). For our small experiments, we randomly select $500$ movies and users from the MovieLens $1m$ data set of $6000$ users' ratings of $4000$ movies \citep{movielens}. For our larger experiments, we use the entire data set.

\item \textbf{Revenue maximization on YouTube.} In the revenue maximization experiment, we choose a set of YouTube users who will each advertise a different product to their network neighbors, and the objective is to maximize product revenue. We adopt an objective function and dataset based on \citep{mirzasoleiman2016fast}. Specifically, the expected revenue from each user is a function $V(S)$ of the sum of influences (edge weights) of her neighbors who are in $S$:

\begin{align}
\label{eqn:youtube_obj}
    f(S) &= \sum_{i \in X } V\Big({ \sum_{j\in S} w_{i,j} }\Big)\\
    V(y)&=y^\alpha
\end{align}

where $X$ is the set of all users (nodes) in the network and $w_{i,j}$ is the network edge weight between users $i$ and $j$, and $\alpha: 0<\alpha<1$ is a parameter that determines the rate of diminishing returns on increased cover. Note that eqn. \ref{eqn:youtube_obj} is a monotone submodular function. We conduct our small experiments on the social network of $50$ randomly selected communities ($\sim500$ nodes) from the $5000$ largest communities in the YouTube social network ~\citep{youtube}. For our larger experiments, we increase this to $2000$ communities ($\sim18000$ nodes).  We set $\alpha=0.9$ for all experiments, and we draw the weights of edges in this network from the uniform distribution $U(1,2)$.

\item \textbf{Influence maximization on a social network.} In this application, we select a set of social network `influencers' to post about a topic we wish to promote, and our objective is to select the set that achieves the greatest aggregate influence. We adopt the following random cover function: an arbitrary social network user has a small independent probability of being influenced by each influencer to whom she is connected, so we maximize the expected count of users who will be influenced by at least one influencer. The probability that a single user $i$ will be influenced is:

\label{eqn:influence}
\[ \begin{cases} 
      f_i(S) &= 1 \text{ \ \ for }  i\in S \\
      f_i(S) &= 1-(1-p)^{|N_S(i)|} \text{ \ \ for }  i\not \in S \\
   \end{cases}
\]
where $|N_S(i)|$ is node $i$'s count of neighbors who are in $S$. We set $p=0.01$. We conduct our small experiments on the CalTech Facebook Network data set \citep{traud2012social} of $769$ Facebook users \& $17000$ edges, and we conduct our larger experiments on the Epinions data set of $27000$ users and $100000$ edges  \citep{nr-aaai15}. 
\end{itemize}

\subsection{Parallelization}
 We parallelize all algorithms via Message Passing Interface (MPI). We make this selection both because it is the industry standard, and also because it allows precise control over the architecture of parallel communication between processors as well as exactly what information is communicated between processors. This allows us to build efficient parallel architectures that minimize communication, such that our implementations are CPU-bound (i.e. query-bound). This property both permits fast implementations also aligns with the theoretical view of adaptive sampling algorithms, which assume that computation time is a function of rounds of queries rather than communication, data copying, etc. In contrast, simpler-to-use parallel libraries (e.g. $JobLib$) often communicate copies of all data to all processors at each communicative step, which may render implementations based on these simpler libraries both slower and also communication-bound.


\subsection{AWS Hardware}
\label{ssec:hardware}
While our MPI implementations of the algorithms are scalable to thousands of cores, we conduct all experiments on an \emph{md5.24xlarge} instance with $96$ cores---the largest single instance currently available on AWS (computing on more cores requires launching an AWS cluster). This instance features Intel Xeon Platinum $8000$ series (Skylake-SP) processors with sustained all core Turbo CPU clock speed of up to $3.1$ GHz.

We select this hardware to ensure both the internal validity and external validity of our experiments. Specifically, with regard to internal validity, we note that if we instead had scaled up to a cluster of multiple instances, then communication times between cores in the same instance vs. across instances may differ to a greater extent. Because different algorithms require different amounts and structures of communication between processors, this might bias runtimes in unpredictable ways.

Second and more importantly, our goal in this paper is \emph{not} to show that \algoptimized \ is the fastest practical algorithm only on large scale state-of-the-art hardware (though larger-scale hardware would increase the runtime advantage of \algoptimized \ over alternatives---see Section \ref{ssec:resultsLTLG}). Instead, our goal is to show that \algoptimized \ is faster than alternatives even with modest hardware that is widely accessible to researchers, and these fast runtimes can be further improved on larger scale hardware. Specifically, we note that this \emph{md5.24xlarge} instance has far fewer cores than the $n$ processors necessary to unlock the full speed potential of \algoptimized, whereas adding more processors cannot accelerate \textsc{Parallel-LTLG} for most experiments due to the fact that \textsc{Parallel-LTLG} has sample complexity less than 95 for many values of $k$ we tried.

\subsection{Instance setup}
We initialize the \emph{md5.24xlarge} instance with Amazon's Deep Learning AMI (Amazon Linux) Version 23.0. We install the \emph{Open-MPI} MPI library on our AWS instance and run all experiments via $ssh$ using $mpirun$ to launch and execute the experiments. 

\subsection{Measuring parallel runtimes.}
\label{ssec:measure_ptime}
We measure true parallel time in the following manner. First, before we start the runtime clock, all processors are initialized with a copy of the objective function and dataset for the experiment, which is followed by a call to a blocking parallel barrier ($comm.barrier()$). This forces the condition that no processor begins computations---and the clock does not start---until all processors are initialized. Runtime is then measured via MPI's parallel clock, $MPI.Wtime()$, from the moment this barrier is completed and the algorithm function is called. Upon the algorithm's completion, we use a blocking parallel barrier ($comm.barrier()$) call to all processors followed by a call to the parallel clock $MPI.Wtime()$. This ensures that a case where one processor finishes its part of the computations early does not result in an erroneously reported lower runtime.

We run all experiments on $95$ cores of the 96-core instance. We deliberately leave one core free during all experiments so that the $95$ cores conducting the experiment are not simultaneously scheduled to run a background task, which may result in a slowdown that would endanger the integrity of measured runtimes.

\subsection{Overview of fast parallel implementations}
\label{ssec:implementations}
For all algorithms, we implement several generally applicable and algorithm-specific optimizations. The intuition behind our generally applicable optimizations is to ensure that (1) implementations are optimized and vectorized such that all operations besides queries take negligible time (such that the algorithms are effectively query-bound); (2) communicative architectures between processors are designed to avoid superfluous communication; (3) we implement parallel reduces where possible to leverage these parallel architectures to further reduce computation; and (4) no algorithm ever queries the marginal value of an element when this marginal value is known to be 0, i.e. $f_T(x), x\in T$. Algorithm-specific optimizations are discussed below.

\subsection{Fast parallel implementation of low-adaptivity algorithms}\label{sec:fast_as}

\begin{itemize}
    \item \textbf{Amortized-Filtering}.
    Recall that when \OPT \ is unknown, one must run \textsc{Amortized-Filtering} once for each guess of \OPT. This requires $\sim 60$ unique runs of \textsc{Amortized-Filtering} to maintain the approximation guarantee even for relatively small $k$ and $\varepsilon=0.1$. Therefore, to optimize this algorithm, we (1) implement binary search over these guesses of \OPT. We also (2) introduce the same stopping condition that we describe for \algoptimized, such that whenever a run of \textsc{Amortized-Filtering} with a particular guess of \OPT \ finds a solution $S$ with $f(S)\ge (1-1/e-0.1)v$, where $v$ is the sum of the $k$ top singletons, then we return this solution. These two optimizations reduce the value of \textsc{Amortized-Filtering}'s solutions in practice, but they dramatically accelerate its runtimes to provide a more stringent runtime benchmark for \algoptimized.
    \item \textbf{Exhaustive-Maximization}. \textsc{Exhaustive-Maximization} tends to run significantly more loop iterations than other benchmarks, so optimizations that reduce these loop iterations result in significant speedup. We note that two loops of this algorithm loop over indices, where indices are calculated as elements of a geometric sequence then rounded to integers. This process results in the algorithm looping over numerous redundant indices, as many unique floating point numbers from these sequences round to the same whole numbers. We therefore achieve large speedups by precomputing these sequences and looping only over unique indices. In addition to this optimization, we also note that the \textsc{Reduced-Mean} subroutine (which is responsible for the vast majority of computation time) requires the processors to parallel-compute a fraction of elements that exceed a threshold. A naive approach would be to computing marginal values in parallel, $gather$ or $allgather$ all of these values (i.e. each processor communicates its share of values to all other processors), and then compute this fraction locally. However, our optimized approach uses a fast parallel reduction where each processor computes its local fraction, then a fast parallel reduce using $MPI.Sum()$ such that (1) processors need only communicate this fraction (float) instead of the entire vector of elements, and (2) the global fraction is then rapidly computed in a parallel reduce. More advanced MPI architectures such as these result in meaningfully lower runtimes in practice, particularly when using relatively fast-to-compute objective functions.
    \item \textbf{Randomized-Parallel-Greedy}. The key decision when implementing \textsc{Randomized-Parallel-Greedy} is how to choose guesses for the step size $\delta$, which is the probability with which we randomly add high-value elements to the solution. This is important because the majority of \textsc{Randomized-Parallel-Greedy}'s runtime is taken up by calls to the multilinear extension that are made in order to choose $\delta$. Specifically, recall that in each iteration, \textsc{Randomized-Parallel-Greedy} calls the multilinear extension to search for the maximum $\delta$ that obeys certain conditions. Here, if we implement \textsc{Randomized-Parallel-Greedy} such that it tests more guesses (i.e. more closely spaced) guesses for $\delta$, then we may find a $\delta$ that is closer to the true maximum $\delta$, but this uses additional calls to the multilinear extension that slow runtimes. We therefore precompute just $n$ guesses for $\delta$ as $[1/n, 2/n, \hdots]$, iterate over each of these, and set $\delta$ to the rightmost value that does not violate the conditions. Based on this choice, the minimum value we attempt for $\delta$ will increase $|S|$ by one element in expectation when all elements are high-valued. By using this relatively large $1/n$ stepsize between subsequent guesses for $\delta$, we reduce solution values in practice, but we accelerate the algorithm (thus providing a more difficult runtime benchmark for \algoptimized). Before making this choice, we also experimented with geometrically-spaced guesses for $\delta$, but found that this resulted in a significant further reduction in performance and also that it caused the algorithm to attempt many very small options for $\delta$ that were unlikely to result in adding a single element.
    
    To run \textsc{Randomized-Parallel-Greedy}, we also require guesses for \OPT. The analysis in \citep{chekuri2018submodular} shows that we can either use multiple guesses for \OPT, or we can use a single guess that is an upper bound for \OPT. We use the latter option, as using a single guess is the fastest approach, so this choice is consistent with our goal of providing the most difficult speed benchmarks for \algoptimized. Specifically, we guess \OPT \ to be the sum $v = \max_{|S| \leq k} \sum_{a \in S} f(a)$ of the $k$ highest valued singletons, which is an upper bound on $\OPT$. We note that this is a tighter upper bound on the value of the true \OPT \ than commonly used alternatives (e.g. $k$ times the value of the top singleton), so by choosing this guess, we further accelerate our runs of \textsc{Randomized-Parallel-Greedy}. We also note that using a single guess for \OPT \ achieves this greater speed by sacrificing some solution value, which is why in our experiments \textsc{Randomized-Parallel-Greedy} sometimes finds solutions that have lower values than other benchmarks.
\end{itemize}

\subsection{Fast parallel implementation of \textsc{Lazier than Lazy Greedy}}\label{sec:pltlg}
\begin{itemize}
\item \textbf{Parallel-LTLG.} Designing a fast implementation of \textsc{Parallel-LTLG} is nontrivial because at each iteration, we want to attempt a lazy update (which requires a single query), but in the event that this lazy update fails, we do not want to complete the iteration any slower than if we had not attempted the lazy update. Put differently, the lazy update attempted at each iteration should result in a speedup when it succeeds, but never in a slowdown when it fails, such that \textsc{Parallel-LTLG} is strictly faster than \textsc{Stochastic-Greedy}.

\smallskip
To accomplish this, we adopt the following optimized parallel architecture for \textsc{Parallel-LTLG}. At each of $k$ rounds, the root processor draws a set $R$ of sample elements from remaining elements in the ground set $X\backslash S$. The root process broadcasts these sample elements to the other processors. Then, all processors simultaneously make a single query: the root process queries the marginal value of the best element according to previous (lazy) marginal values, and remaining processors each query a single other element from the sample. If the root process succeeds in finding a lazy update with its single query, it communicates this to all processors, and all processors add this element to $S$ and move to the next iteration. If the lazy update fails, then the $c$ processors have \emph{already} completed $c$ marginal value queries of the samples (so no time is lost). They then simply each compute $1/c$ of the $(|R|-c)$ \emph{remaining} samples' marginal values, communicate them, add the best element to $S$, and move to the next iteration.
\end{itemize}

\subsection{Experiment set 2 results:  additional discussion}
\label{ssec:LTLGqueries}
\begin{figure}[t]
\centering
\includegraphics[height=0.480\textwidth]{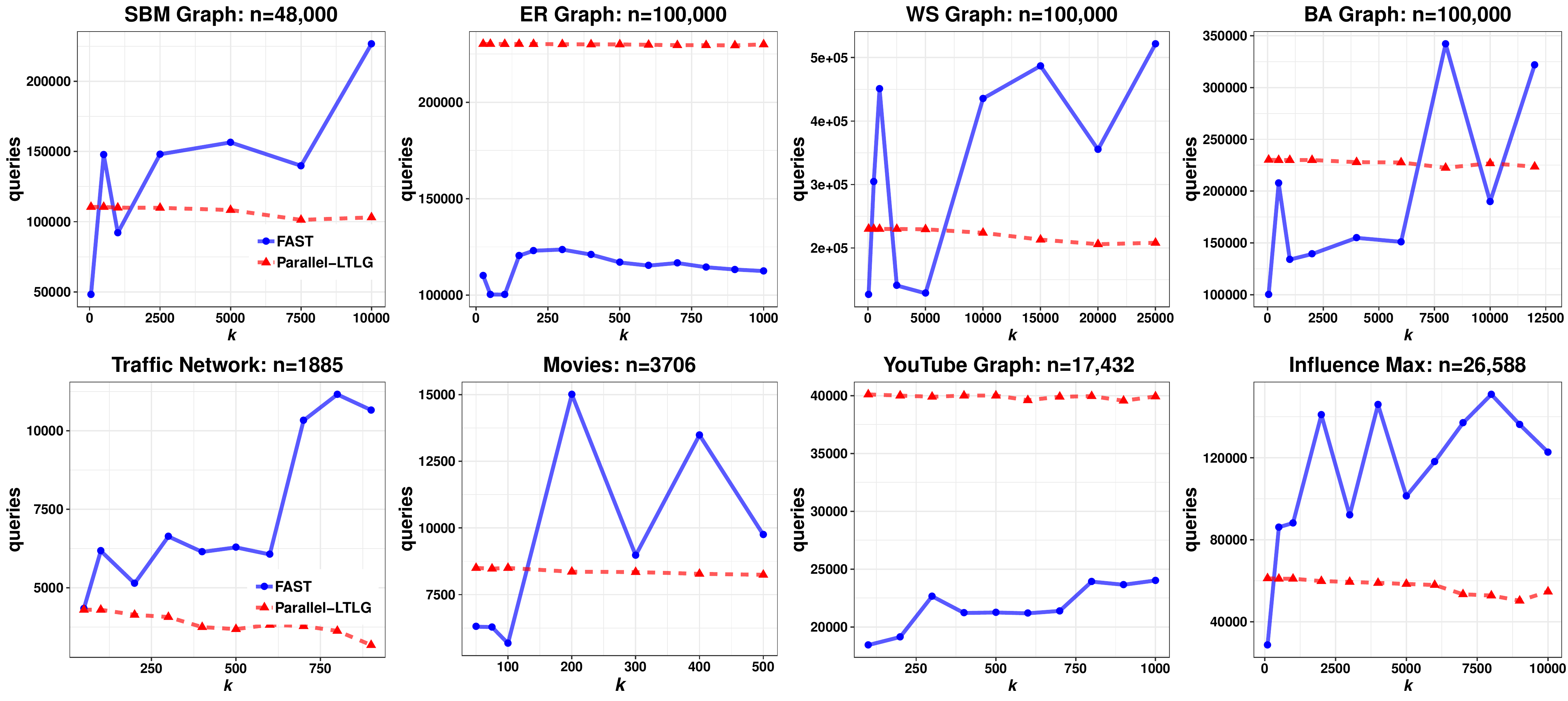}
\caption{\textit{Experiment Set 2: Queries used by \algoptimized \ (blue) \ vs. \textsc{Parallel-LTLG} (red).}}
\label{fig:queries}
\end{figure}

Figure \ref{fig:queries} plots queries used by \algoptimized \ and \textsc{Parallel-LTLG} for the 8 objectives in \emph{Experiment set 2}. When counting queries for \textsc{Parallel-LTLG}, we count rounds where a lazy update succeeded as single query rounds despite the fact that in this case \textsc{Parallel-LTLG} uses $\max[c, s]$ queries where $c$ is the number of processors and $s$ is its sample complexity per round. This choice allows us to compare the queries used by \algoptimized \ vs. \textsc{Parallel-LTLG} for any $k$ and determine whether \algoptimized \ used fewer queries than \emph{serial} $LTLG$ would use. Note that this occurs for various $k$ in $7$ of the $8$ experiments. We also note that in practice, serial \textsc{LTLG} is often slower than \algoptimized \ even when both perform the same number of queries due to the fact that \algoptimized \ performs more queries at a time (i.e. more per round for fewer rounds), which is often computationally faster.


%


\end{document}